\PassOptionsToPackage{dvipsnames, table, xcdraw}{xcolor}
\documentclass[]{usc}
\usepackage{url}

\usepackage{amsthm}
\usepackage{minitoc}

\usepackage{array}
\definecolor{linkColor}{HTML}{E74C3C}
\definecolor{pearcomp}{HTML}{B97E29}
\definecolor{citeColor}{HTML}{2980B9}
\definecolor{urlColor}{HTML}{1D2DEC}
\definecolor{conjColor}{HTML}{9ab569}
\definecolor{jigsawGreen}{HTML}{1f78b4}
\definecolor{promptblue}{RGB}{44,95,153}
\definecolor{promptred}{RGB}{180,30,30}
\definecolor{headergray}{RGB}{60, 60, 60} 
\definecolor{targetred}{HTML}{2ca25f} 
\definecolor{rowgray}{RGB}{245, 245, 245}
\lstdefinelanguage{prompt}{
  basicstyle=\ttfamily\small,
  showstringspaces=false,
  breaklines=true,
  columns=fullflexible
}

\usepackage{amsmath,amssymb}
\usepackage{bm}
\usepackage[capitalise]{cleveref}
\crefname{figure}{Figure}{figures}
\Crefname{figure}{Figure}{Figures}
\crefname{table}{Table}{tables}
\Crefname{table}{Table}{Tables}
\crefname{equation}{Equation}{equations}
\Crefname{equation}{Equation}{Equations}
\usepackage{mathtools}
\usepackage{stmaryrd}
\usepackage{wrapfig}
\usepackage{microtype}
\usepackage{graphicx}
\usepackage{multirow}
\usepackage{rotating} % For text rotation
\usepackage{xspace}
\usepackage{colortbl}
\newcommand{\redtext}[1]{\textcolor{targetred}{#1}}
\usepackage{lmodern} % Use vector fonts
\usepackage{float}
\tikzset{
  invisible/.style={opacity=0},
  visible on/.style={alt={#1{}{invisible}}},
  alt/.code args={<#1>#2#3}{
    \alt<#1>{\pgfkeysalso{#2}}{\pgfkeysalso{#3}}
  },
}

\newtheorem{definition}{\textbf{Definition}}

\newtheorem{proposition}{\textbf{Proposition}}

\newcommand{\cX}{\mathcal{X}}
\newcommand{\cY}{\mathcal{Y}}

\newcommand{\cD}{\mathcal{D}}

\newcommand{\cL}{\mathcal{L}}

\newcommand{\R}{\mathbb{R}}
\DeclareMathOperator*{\VecOp}{vec}
\DeclareMathOperator*{\MatOp}{mat}

\usepackage{bbold}

\usepackage[normalem]{ulem}
\usetikzlibrary{shapes, positioning}

\newcommand{\methodName}{\textsc{CrispEdit}} 
\newcommand{\ourmethod}{\textcolor{black}{\methodName}\xspace}
\newcommand{\ourmethodseq}{\textcolor{black}{\methodName-\textsc{Seq}}\xspace}
\newcommand{\llama}{\textcolor{black}{LLaMA-3-8B-Instruct}\xspace}
\newcommand{\qwen}{\textcolor{black}{Qwen-2.5-1.5B-Instruct}\xspace}

\renewcommand{\cite}[1]{\citep{#1}}

\usepackage{color}
\definecolor{cm}{RGB}{0,0,200}
\definecolor{purple}{RGB}{200,0,200}

\usepackage[intoc]{nomencl}
\makenomenclature

\makeatletter
\newcommand{\vast}{\bBigg@{2.5}}
\newcommand{\Vast}{\bBigg@{5}}
\makeatother

\usepackage{bbm}

\usepackage[utf8]{inputenc}
\usepackage[LGR,T1]{fontenc}

\usepackage{breqn}
\usepackage{enumitem,kantlipsum}

\definecolor{rliableolive}{HTML}{BBCC33}
\definecolor{rliableblue}{HTML}{77AADD}
\definecolor{rliablered}{HTML}{EE8866}

\definecolor{takeawaycolor}{HTML}{FFF1E6}
\definecolor{takeawaycolor2}{HTML}{F2F6FF}
\definecolor{takeawaycolor3}{HTML}{EAF6EF}
\definecolor{takeawaycolor4}{HTML}{E6F4F1}
\definecolor{takeawaycolor5}{HTML}{FFF9E5}
\definecolor{takeawayborder}{HTML}{FFD8C2}
\definecolor{takeawayheader}{HTML}{0B3C8C}

\tcbset{
  aibox/.style={
    width=\linewidth,
    top=9pt,
    bottom=4pt,
    colback=takeawaycolor5!100!white,
    colframe=black!75!white,
    colbacktitle=black!75!white,
    boxrule=0.75pt,
    enhanced,
    center,
    attach boxed title to top left={yshift=-0.1in,xshift=0.15in},
    boxed title style={boxrule=0pt,colframe=white,},
  }
}
\newtcolorbox{AIbox}[2][]{aibox,title=#2,#1}

\tcbset{
  greenaibox/.style={
    width=\linewidth,
    top=9pt,
    bottom=4pt,
    colback=takeawaycolor3!100!white,
    colframe=black!75!white,
    colbacktitle=black!75!white,
    boxrule=0.75pt,
    enhanced,
    center,
    attach boxed title to top left={yshift=-0.1in,xshift=0.15in},
    boxed title style={boxrule=0pt,colframe=white,},
  }
}
\newtcolorbox{greenAIbox}[2][]{greenaibox,title=#2,#1}

\usepackage{algorithm}
\usepackage{algorithmic}
\usepackage{enumitem,kantlipsum}

\usepackage{booktabs}
\usepackage{wrapfig}

\tcbuselibrary{skins}
\usepackage{tabularx}

\title{\fontsize{20pt}{24pt}\selectfont CrispEdit: Low-Curvature Projections for \\Scalable Non-Destructive LLM Editing}

\author{Zarif Ikram} \author{Arad Firouzkouhi} \author{Stephen Tu} \author{Mahdi Soltanolkotabi} \author{Paria Rashidinejad}

\affiliation{University of Southern California}

\abstract{ 
A central challenge in large language model (LLM) editing is capability preservation: methods that successfully change targeted behavior can quietly game the editing proxy and corrupt general capabilities, producing degenerate behaviors reminiscent of {proxy/reward hacking}.  We present \ourmethod{}, a scalable and principled second-order editing algorithm that treats capability preservation as an explicit constraint, unifying and generalizing several existing editing approaches. \ourmethod{} formulates editing as constrained optimization and enforces the constraint by projecting edit updates onto the low-curvature subspace of the capability-loss landscape. At the crux of \ourmethod{} is expressing capability constraint via \emph{Bregman divergence}, whose quadratic form yields the Gauss–Newton Hessian exactly and even when the base model is not trained to convergence. We make this second-order procedure efficient at the LLM scale using Kronecker-factored approximate curvature (K-FAC) and a novel \emph{matrix-free projector} that exploits Kronecker structure to avoid constructing massive projection matrices. Across standard model-editing benchmarks, \ourmethod{} achieves high edit success while \textbf{keeping capability degradation below 1\%} on average across datasets, significantly improving over prior editors.}

\date{February 17, 2026}
\website{https://crispedit.github.io}
\code{https://github.com/zarifikram/CrispEdit}
\emails{\email{\{zikram,firouzko,stephen.tu,soltanol,paria.rashidinejad\}@usc.edu}  }

\begin{document}

\maketitle
\section{Introduction}
\label{section:intro}

\begin{wrapfigure}[21]{r}{0.45\textwidth}
    \centering
    \includegraphics[width=\linewidth]{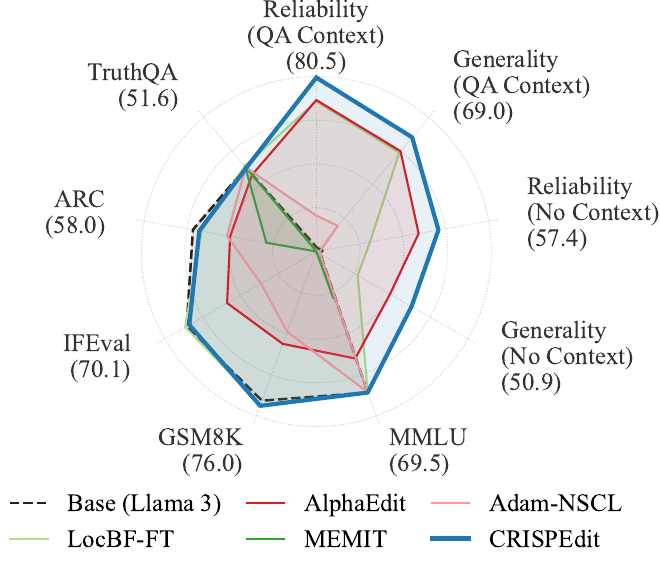}
    \caption{\textbf{Comparison overview of \ourmethod.} \ourmethod achieves strong edit reliability and generality, with and without QA context, while preserving broad base capabilities (MMLU, GSM8K, IFEval, ARC-C, TruthfulQA) on \llama.}
    \label{fig:spider_plot}

\end{wrapfigure}
Large language models (LLMs) are rapidly becoming a shared backbone for knowledge work, spanning search and question answering~\citep{gao2023retrieval, lewis2020retrieval}, science~\citep{jumper2021highly}, software development~\citep{chen2021evaluating}, decision support~\citep{lopez2023can}, and education~\citep{kasneci2023chatgpt}. Yet every day, facts shift, new discoveries land, products ship, hallucinations or unsafe behaviors are uncovered, quickly making the models stale. Retraining from scratch is the cleanest way to absorb this drift, but it is also the most expensive and slowest. Model editing~\citep{sinitsineditable, de2021editing,mitchell2022memory, wang2024knowledge} offers a practical alternative: apply targeted updates to correct a fact, insert new knowledge, remove unsafe behavior, personalize the style \emph{while leaving everything else intact}.

\begin{figure}[h]
    \centering
    \includegraphics[width=0.9\linewidth]{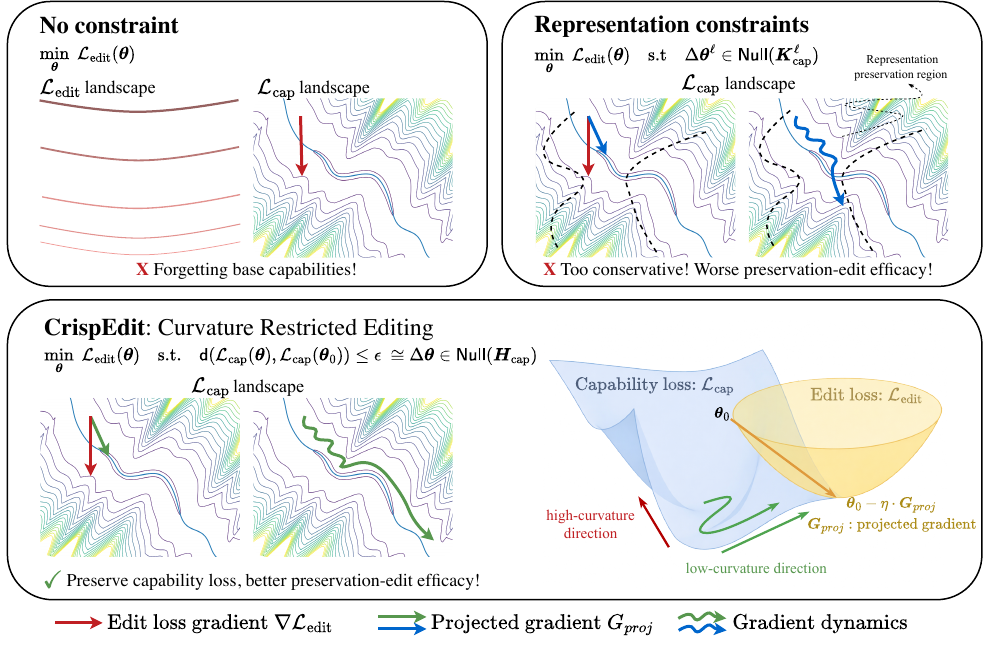}
    \caption{\textbf{Geometric interpretation of \ourmethod compared to baseline editing strategies.} \textit{Top left:} Standard gradient descent effectively minimizes edit loss but moves perpendicular to the capability contours, resulting in high capability loss (degradation). \textit{Top right:} Projecting onto the nullspace of activation covariance is overly conservative; it preserves representations but restricts the update too heavily to successfully optimize the edit loss. \textit{Bottom:} \ourmethod projects the update onto the low-curvature subspace of the capability loss. This allows changes in representations to satisfy the edit while moving along the ``valley'' of the landscape to maintain general model capabilities.}
    \label{fig:main_header}
\end{figure}
In many cases, edits may appear to succeed while quietly degrading broader capabilities reminiscent of reward/proxy hacking \cite{gao2023scaling}. This degradation can manifest as brittle reasoning, weaker instruction-following, or even broken fluency. In response, prior work has introduced heuristic guardrails: restrict updates to a small set of parameters~\citep{hu2022lora, yu2024melo}, localize ``where the knowledge lives,''~\citep{meng2022locating, yang2025finetuning, gu2025ultraedit} or constrain representation changes (e.g., via subject-centric ``knowledge vectors'')~\citep{meng2023massediting, fang2025alphaedit}. Despite improvements, these methods tend to bake in strong assumptions about edit structure (e.g., explicit subjects/entities) and impose constraints in parameter or representation space that are only indirectly tied to capability preservation, resulting in a poor edit–preservation trade-off. Indeed, editors built on such constraints still perform poorly when tested \emph{in the wild} under natural autoregressive generation, despite looking strong under unrealistic teacher-forced evaluation that scaffolds the ground-truth prefix and target length~\citep{yang-etal-2025-mirage}.

In this paper, we adopt a first-principles formulation of model editing: an edit should reduce an edit loss while leaving broader capabilities effectively unchanged. Accordingly, we pose editing as a constrained optimization problem that seeks to minimize the edit loss subject to negligible changes in capability loss, measured on a designated capability set via a distance metric~(\Cref{sec:problem_formulation}). Standard approaches that replace the constraint with a soft penalty typically require nontrivial tuning and can be prohibitively expensive when the capability set is far larger than the edit set. This motivates us to ask: \emph{How to enforce capability preservation directly, without turning editing into full retraining?} 

To address this, we introduce \ourmethod{} ({C}urvature-{R}estricted {I}n-{S}itu {P}arameter {E}diting), a scalable non-destructive editor, built around the following core ideas.

\textbf{1. Preserving capabilities with low-curvature projections.} A core idea behind \ourmethod{} is that not all parameter directions are equally important for preserving a model's capabilities. Recent work shows that the curvature of the pretrained loss landscape can be characterized by the Hessian, which is observed to be highly anisotropic: sharp in a small number of directions and flat in others~\citep{sagun2017empirical, oymak2019generalization, pmlr-v97-ghorbani19b, kalra2026scalable}. \ourmethod exploits this structure by projecting updates into low-curvature subspaces of Hessian, effectively ``hiding'' parameter movement where capabilities are minimally affected~(see \cref{fig:main_header} and \Cref{sec:method_hessian}).

\textbf{2. Avoiding base model convergence requirement with Bregman constraint.} A quadratic approximation based on the standard Hessian---which instantiates our formulation with a Euclidean distance---requires assuming that the base model is trained to (near)-convergence, which is rarely satisfied in practice for modern large networks. We resolve this by measuring capability preservation using \emph{Bregman divergence}. This choice yields a quadratic form expressed exactly in terms of the \emph{Gauss-Newton Hessian} (GNH), even when the base model is not trained to convergence, avoiding stationarity assumptions.

\textbf{3. Representation constraints as a restrictive special case.} Our Bregman-GNH formulation also sheds light on several successful prior heuristics. We prove (see \Cref{prop:alphaedit-special-case}) popular editors such as AlphaEdit~\citep{fang2025alphaedit} and Adam-NSCL~\citep{wang2021training} solve an approximate special case of our framework, but do so within \emph{far more restrictive} and lower-dimensional subspaces, leading to a worse capability preservation-edit tradeoff (\Cref{fig:spider_plot}). 

\textbf{4. Scalable matrix-free low-curvature projectors.} The remaining challenge is scale: how can we efficiently compute and store curvature information for billion-parameter transformers? \ourmethod addresses this with two key ideas:
\begin{enumerate}[label=(\alph*), leftmargin=*, nosep]
    \item The resulting GNH is amenable to accurate approximations via Kronecker-factored approximate curvature~\citep[K-FAC]{martens2015optimizing}, which we leverage to enable efficient computation of the low-curvature projection matrix.
    \item Instead of explicitly constructing a low-curvature projection matrix, we introduce (\Cref{sec:matrix-free projection}) a matrix-free projector that exploits the Kronecker eigen-structure: rotate gradients into a factored eigenbasis, mask high-curvature components, and rotate back. This makes constraint-aware second-order editing feasible and enables precomputing capability curvature statistics once and reusing them across many future edits, amortizing cost and enabling batch and sequential editing.
\end{enumerate}

\noindent\textbf{Experimental results.} We evaluate \ourmethod in both small- and large-scale regimes. In controlled small-scale experiments on image classification (MNIST $\mapsto$ FashionMNIST), where calculating exact curvature is feasible, we show that Hessian low-curvature projections yield the strongest capability preservation, and that K-FAC closely tracks this behavior cheaply. We then scale \ourmethod to edit LLMs (e.g., \llama and \qwen) and evaluate them as used in practice: edits should be \emph{reliable} in standalone autoregressive generations, \emph{generalize} across semantically equivalent in-scope queries, and remain \emph{local}, preserving out-of-scope knowledge and broad skills such as reasoning, instruction-following, and truthfulness. We further test our method in both \textit{batch} editing, where many edits are applied at once, and \textit{sequential} editing, where batches of edits are applied to the model sequentially. Across settings, \ourmethod consistently improves the edit–capability trade-off, achieving strong edit success while substantially reducing capability degradation, with modest compute and storage requirements.\footnote{Using cached curvature, 3000 edits with our method takes six minutes on an NVIDIA A40 GPU.}

\section{The Model editing problem}\label{sec:problem_formulation}
Let $f_{\bm{\theta}}: \cX \mapsto \cY$ denote a model with parameters $\bm{\theta} \in \Theta \subseteq \mathbb{R}^p$, mapping inputs $x\in \cX$ to outputs $y \in \cY$. 
Model editing seeks to update a pretrained (base) model $f_{\bm{\theta}_0}$ with initial parameters $\bm{\theta}_0$, using a provided edit target pair $(x, y) \in \cX \times \cY$, while preserving the existing capabilities of the base model. We formalize this as follows.

Let $\cD_{\text{cap}} = \{(x_i, y_i)\}_{i=1}^{n}$ be a reference dataset that serves as a proxy for capabilities we wish to preserve, an exemplar of the domains on which the model should continue to perform well. We formulate capability preservation through the empirical loss
$$\mathcal{L}_{\text{cap}}(\bm{\theta}; \mathcal{D}_{\text{cap}})
:= \frac{1}{n}\sum_{i=1}^{n} \ell \; \!\big(f_{\bm{\theta}}(x_i),y_i\big),$$
where $\ell(\hat y, y)$ is a task-appropriate loss (e.g., cross entropy).
Preserving capabilities then means keeping $\mathcal{L}(\bm{\theta}; \mathcal{D}_{\text{cap}})$ close to its pre-edit value, i.e., $\mathcal{L}(\bm{\theta}; \mathcal{D}_{\text{cap}}) \approx \mathcal{L}(\bm{\theta}_0; \mathcal{D}_{\text{cap}})$. Let $\cD_{\text{edit}} = \{(x_i, y_i)\}_{i=1}^{T}$ be the edit dataset containing the desired edit pairs. We write \(\mathcal{L}_{\text{edit}}(\bm{\theta};\mathcal{D}_{\text{edit}})\) to denote an edit loss, such as the negative log-likelihood of edit outputs. 
Using the language of constrained optimization, a
natural optimization problem that expresses our desire
to minimize edit loss subject to preserving capabilities is the following:\footnote{We will drop the dependency of the capabilities and edit losses on datasets $\mathcal{D}_{\text{edit}}$ and $\mathcal{D}_{\text{cap}}$ when it is clear from the context.}
\begin{align}\label{eq:batch_constrained_optimization}
\min_{\bm{\theta} \in \Theta} \ \mathcal{L}_{\text{edit}}(\bm{\theta})
 \quad \text{s.t.} \quad \mathsf{d} \left(\mathcal{L}_{\text{cap}}(\bm{\theta}), \mathcal{L}_{\text{cap}}(\bm{\theta}_0) \right)
\le \varepsilon,
\end{align}
where $\mathsf{d}(\cdot,\cdot)$ is a measure of distance, such as the difference between the two loss values or the Bregman divergence, and $\varepsilon$ is a small tolerance value. The above formulation is general, unifying and extending many existing model editing frameworks as we discuss in \Cref{sec:related_work}.

While Problem \eqref{eq:batch_constrained_optimization} rigorously expresses our desired intent for model editing, actually solving \eqref{eq:batch_constrained_optimization}, especially at LLM scale, is challenging due to the hard constraint. We note that we focus on the constrained formulation above in lieu of the standard Lagrangian relaxation to \eqref{eq:batch_constrained_optimization}, namely 
$\min_{\bm{\theta} \in \Theta} \ \mathcal{L}_{\text{edit}}(\bm{\theta}) + \lambda\mathsf{d} \left(\mathcal{L}_{\text{cap}}(\bm{\theta}) , \mathcal{L}_{\text{cap}}(\bm{\theta}_0) \right)$.
This is due to the fact that in typical operating regimes
$n$ (the number of reference pairs) far exceeds $T$ (the number of edits), and the computational overhead of gradient-based optimization on the unconstrained problem can be non-trivial. We avoid this complexity by considering an alternative approach to approximating \eqref{eq:batch_constrained_optimization} based on low-curvature projections.

\section{\ourmethod: Curvature-Restricted In-Situ Parameter Editing}\label{sec:method}

We now present \ourmethod{} for solving Problem \eqref{eq:batch_constrained_optimization}.
The key idea is to edit only along directions that are locally ``safe'' for maintaining capabilities as informed by the constraint. We start in \Cref{sec:method_hessian} with a simple instantiation of \ourmethod{} under the standard capability loss difference and derive a principled curvature-restricted model-editing algorithm. Then, in \Cref{sec:method_bregman}, we leverage \emph{Bregman divergences} to derive a practical editing approach that scales to billion-parameter LLMs.

For what follows, we assume that both the maps $\hat y \mapsto \ell(\hat y, y)$ and $\bm{\theta} \mapsto f_{\bm{\theta}}({x})$ are twice continuously differentiable over their respective domains. This immediately holds for architectures with smooth activation functions such as GeLU/SwiGLU. Furthermore, this assumption can readily be relaxed to functions that are twice differentiable except on a measure zero set, such as architectures with ReLU activations; for simplicity of exposition, we omit these details.

\subsection{Preserving capabilities with low-curvature updates}
\label{sec:method_hessian}

We first consider the distance measure to be the standard distance $\mathsf{d}(a, b) = |a-b|$. Furthermore, in this section, we assume that the base parameters $\bm{\theta}_0$ are a local minima of the capabilities loss $\cL_{\text{cap}}(\bm{\theta})$; we remove this assumption in \Cref{sec:method_bregman} by using a different distance measure. Applying a second-order Taylor expansion to the constraint in \eqref{eq:batch_constrained_optimization} yields $\mathcal{L}_{\text{cap}}(\bm{\theta}) - \mathcal{L}_{\text{cap}}(\bm{\theta}_0) \approx \frac{1}{2} (\bm{\theta} - \bm{\theta}_0)^\top \bm{H}_{\text{cap}} (\bm{\theta} - \bm{\theta}_0)$, where $\bm{H}_{\text{cap}} \coloneqq \nabla^2_\theta \cL_{\text{cap}}(\bm{\theta}_0)$ is the Hessian of the capability loss function evaluated at $\bm{\theta}_0$, and the first term in Taylor expansion is zero because $\nabla_\theta \cL_{\text{cap}}(\bm{\theta}_0) = 0$. Under this setting, \eqref{eq:batch_constrained_optimization} can be approximated by optimizing the following quadratically constrained optimization problem:
\begin{align}\label{eq:hessian_constrained_problem}
    \min_{\bm{\theta} \in \Theta} \; \mathcal{L}_{\text{edit}}(\bm{\theta}) \quad \text{s.t.} \quad (\bm{\theta} - \bm{\theta}_0)^\top \bm{H}_{\text{cap}} (\bm{\theta} - \bm{\theta}_0) \leq \varepsilon.
\end{align}

In the deep learning literature, it is well-understood that
in typical overparameterized settings, the Hessian $\bm{H}_{\text{cap}}$ at the end of training is usually low-rank~\cite{sagun2017empirical,oymak2019generalization,pmlr-v97-ghorbani19b}. Thus, the ellipsoidal constraint in \eqref{eq:hessian_constrained_problem} offers many parameter directions around $\bm{\theta}_0$ of \emph{low-curvature}, where the capability loss $\cL_{\text{cap}}$ remains (approximately) invariant. These low-curvature directions enable the optimization \eqref{eq:hessian_constrained_problem} to decrease the edit loss $\cL_{\text{edit}}$, while limiting loss of capabilities. Furthermore, compared to the Lagrange relaxation objective, the quadratic constraint offers several key advantages:
\begin{enumerate}[label=(\alph*), leftmargin=*, nosep]
    \item \textit{Strict control of capability loss:} The ellipsoidal constraint can be enforced
    via projected gradient or trust-region methods, 
    enabling strict control of tolerated capability degradation; we discuss this shortly.
    \item \textit{Scalability to billion-parameter models:} The second-order relaxation of the constraint {forms the foundation for efficiently scaling} our approach to billion-parameter LLMs leveraging {Bregman divergences} (cf.~\Cref{sec:method_bregman}).
    \item \textit{Pre-computation:} The {curvature model $\bm{H}_{\text{cap}}$ can be precomputed once} and reused across many subsequent edits, amortizing cost and enabling sequential and online interventions (cf.~\Cref{sec:sequential_updates}).
\end{enumerate}

\textbf{Projected low-curvature gradient descent.} We can enforce the constraint in \eqref{eq:hessian_constrained_problem} by ensuring that the weight changes $\Delta \bm{\theta} = \bm{\theta} - \bm{\theta}_0$ are in the (approximate) null-space of the Hessian $\bm{H}_{\text{cap}}$, i.e., $\bm{H}_{\text{cap}} \Delta \bm{\theta} \approx 0$ which is equivalent to $\Delta \bm{\theta} \in \mathsf{Null}(\bm{H}_{\text{cap}})$. A sufficient condition to enforce the constraint during gradient descent is projecting the gradients to the approximate null-space of $\bm{H}_{\text{cap}}$ at every gradient step. 

Let $\bm{H}_{\text{cap}} = \bm{U \Sigma U}^\top$ be the eigen-decomposition of $\bm{H}_{\text{cap}}$, where $\bm{\Sigma} = \mathsf{diag}(\sigma_1, \dots, \sigma_p)$ and $\sigma_1 \geq \dots \geq \sigma_p \geq 0$ (recall that $\bm{\theta}_0$ is locally optimal). We construct a low-curvature projector by discarding the top eigenspace. Concretely, given an energy threshold $\gamma \in (0,1)$, let $k := \min \{r \in [p] \mid \sum_{i=1}^r \sigma_i /  \sum_{i=1}^p \sigma_i \geq \gamma\}$ denote the minimum index capturing
$\gamma$-fraction of the eigenspectrum. Then, the orthogonal projection to the \emph{remaining} directions $\bm{U}_{>k} \coloneqq [u_{k+1} \mid \dots \mid u_p]$ can be computed as:
\begin{align}\label{eq:projected_gradients}
    \bm{g}_t^{\text{proj}} = \bm{P}_\gamma \nabla_{\bm{\theta}} \cL_{\text{edit}}(\bm{\theta}_t), \quad \text{where} \quad  \bm{P}_\gamma := \bm{U}_{>k} \bm{U}_{>k}^\top.
\end{align}

Intuitively, the projection in \eqref{eq:projected_gradients} removes the components of the edit gradient that point in the directions where capability loss is sensitive. We will refer to the subspace spanned by $\bm{U}_{>k}$ as the $\gamma$-approximate nullspace.

\subsection{Gauss-Newton constraint via Bregman divergence}\label{sec:method_bregman}

In \Cref{sec:method_hessian} and deriving \eqref{eq:hessian_constrained_problem}, we assumed that $\nabla_{\bm{\theta}}\cL_{\text{cap}}(\bm{\theta}_0) = 0$. However, in training neural networks, especially LLMs, one typically does not train the network to convergence, to avoid overfitting. Moreover, the capability loss can only be viewed as a mere \textit{proxy} to the pretraining loss. To avoid relying on the linear term vanishing,
we instantiate \ourmethod{} using a \emph{Bregman divergence} that is always first-order flat at $\bm{\theta}_0$.

\begin{definition}[\textbf{Bregman divergence}]\label{eq:bregman_def} For a pair $(x,y)$ and loss $\ell(\cdot, \cdot)$, define the Bregman divergence: 
\begin{align*}%
\begin{split}
\mathsf{d}^{\text{Breg}}_{\ell,y}\!\big(f_{\bm{\theta}}(x), & f_{\bm{\theta}_0}(x)\big)
:=\;
 \ell(f_{\bm{\theta}}(x),y) - \ell(f_{\bm{\theta}_0}(x),y) - \left\langle \nabla \ell(f_{\bm{\theta}_0}(x),y),\, f_{\bm{\theta}}(x)-f_{\bm{\theta}_0}(x)\right\rangle.
\end{split}
\end{align*}
\end{definition}

With this definition, we now consider a distance defined as $\mathsf{d}(\mathcal{L}_{\text{cap}}(\bm{\theta}), \mathcal{L}_{\text{cap}}(\bm{\theta}_0)) := \frac{1}{n} \sum_{i=1}^{n} \mathsf{d}^{\text{Breg}}_{\ell,y_i} (f_{\bm{\theta}}(x_i),  f_{\bm{\theta}_0}(x_i))$. A key property of Bregman divergence is that in the second-order Taylor approximation, the gradient is zero for any fixed $\bm{\theta}$, resulting in the following (cf. Appendix~\ref{appendix:bregman_divergence}):
\begin{align*}
    \mathsf{d}^{\text{Breg}}_\ell(\mathcal{L}_{\text{cap}}(\bm{\theta}), \mathcal{L}_{\text{cap}}(\bm{\theta}_0))
    &\approx \frac{1}{2} (\bm{\theta} - \bm{\theta}_0)^\top \bm{G}_{\text{cap}} (\bm{\theta} - \bm{\theta}_0),
\end{align*}
where $\bm{G}_{\text{cap}}$ is the Gauss-Newton Hessian (GNH, also referred to as the Generalized Gauss-Newton), defined as $\bm{G}_{\text{cap}} := \mathbb{E}_{\cD_{\text{cap}}} \left[\bm{J}^\top \bm{H}_{\hat y} \bm{J} \right]$. Here, $\bm{J} = \nabla_{\bm{\theta}} f_{\bm{\theta}}(x)$ is the network's parameter-output Jacobian, and $\bm{H}_{\hat y} = \nabla^2_{\hat y} \ell$ is the Hessian of the loss with respect to the network's outputs, with the expectation taken empirically over the dataset $\cD_{\text{cap}}$. Importantly, $\bm{G}_{\text{cap}}$ is well-behaved for overparameterized and partially trained networks, and lends itself to reliable and scalable approximations which we explore below.

\textbf{Connections to existing model editing methods.} It turns out many existing heuristic model editing methods can be viewed as solving the problem \eqref{eq:hessian_constrained_problem} via conservative approximations of the quadratic constraint, and with more restrictive assumptions. For example, the popular AlphaEdit technique~\citep{fang2025alphaedit} (and related methods like Adam-NSCL~\citep{wang2021training}) can be viewed as solving the following approximate optimization problem:
\begin{align}
    \label{eq:alphaedit_problem}
    \begin{split}
    &\min_{\bm{\theta}} \; \mathcal{L}_{\text{edit}}(\bm{\theta}) \quad 
    \text{s.t.} \quad \bm{\theta} - \bm{\theta}_0 \in \mathsf{Null}(\bm{K}_{\text{cap}}).
    \end{split}
\end{align}
Here, matrix $\bm{K}_{\text{cap}}$ is constructed from the so-called \textit{knowledge vectors} for \emph{a particular} MLP layer, used for preserving capabilities in certain domains of interest. We show that AlphaEdit solves a special, more restrictive problem compared to our approach; the proof can be found in Appendix~\ref{sec:proofs_of_generalization}. 

\begin{proposition}[\textbf{AlphaEdit is more conservative}]
\label{prop:alphaedit-special-case}
Fix an MLP layer $l$ and consider updating only the weights of layer $l$. Let $\bm{K}^l_{\text{cap}} \coloneqq \bm{I} \otimes [\bm{a}_{l-1}^1, \dots, \bm{a}_{l-1}^n]$ be the layer-input activations on the capability dataset, and $\bm{G}^l_{\text{cap}}$ be the GNH. Then, $\mathsf{Null}( \bm{K}^l_{\text{cap}}) \subseteq \mathsf{Null}(\bm{G}^l_{\text{cap}})$.
\end{proposition}

Unlike AlphaEdit's representation-level restriction via $\bm{K}_{\text{cap}}$, our method preserves capabilities through loss curvatures via $\bm{G}_{\text{cap}}$. Furthermore, our approach can update multiple layers simultaneously, whereas AlphaEdit edits one layer at a time; consequently, a direct comparison requires matching the edited parameter subset. \Cref{prop:alphaedit-special-case} shows that even if we artificially restrict our method to a single layer $l$, the feasible update subspace defined by the corresponding layerwise GNH is a \emph{superset} of AlphaEdit's layerwise subspace. We emphasize that this constraint of the form can be significantly more restrictive than our approach. In particular, $\mathsf{Null}(\bm{K}_{\text{cap}})$ can be a subspace of \emph{much smaller dimension} than the nullspace of the GNH. Furthermore, in contrast to the knowledge matrix, in practice, the GNH is known to be flat in many directions, e.g., due to network overparameterization \cite{sagun2017empirical,oymak2019generalization}. Therefore, the constraint in AlphaEdit can be significantly more restrictive, leading to a worse tradeoff between preserving prior capabilities and applying the new edits, as evidenced by our comparative analysis in \cref{sec:llm_experiments}.

\begin{greenAIbox}{Result: Representational constraint is a restrictive special case of our formulation}
We prove that heuristic methods like AlphaEdit enforce updates within the nullspace of layer inputs ($\bm{K}_{\text{cap}}$), which is a strict subset of the nullspace of the loss curvature ($\bm{G}_{\text{cap}}$) utilized by our method. Consequently, AlphaEdit solves a significantly more constrained optimization problem, limiting its accessible parameter space and resulting in a worse tradeoff between editing efficacy and capability preservation.
\end{greenAIbox}

\subsection{K-FAC for scalable, matrix-free projections}
\label{sec:matrix-free projection}
The remaining obstacle is scale: $\bm{G}_{\text{cap}}$ is expensive to compute and represent as a matrix. To address this, we approximate $\bm{G}_{\text{cap}}$ with Kronecker-Factored Approximate Curvature (K-FAC)\footnote{While K-FAC approximates the Fisher information matrix, for many models, such as the transformers with softmax output and cross-entropy loss, it is equivalent to the GNH~\citep{martens2020naturalgradient}.}~\citep{martens2015optimizing, george2018fast}. At a high level, K-FAC approximates $\bm{G}_{\text{cap}}$ as a block-diagonal matrix, i.e., $\bm{G}_{\text{cap}} \approx \mathrm{blkdiag}( \bm{G}^{1}_{\text{cap}}, \dots, \bm{G}^{L}_{\text{cap}} )$ for a network with $L$ layers. To describe each block-diagonal approximation, suppose that layer $l$ of an MLP computes its outputs as follows: $\bm{s}_l = \bm{W}_l \bm{a}_{l-1}$ and $\bm{a}_l = \phi_l(\bm{s}_l)$, where $\bm{a}_{l-1} \in \mathbb{R}^{d_{\text{in}}}$ are input activations, $\bm{W}_l \in \mathbb{R}^{d_{\text{out}} \times d_{\text{in}}}$ are layer weights (including any bias terms), and $\bm{s}_l \in \mathbb{R}^{d_{\text{out}}}$ are layer pre-activations. Let $\bm{g}_l = \nabla_{\bm{s}_l} \log p(\hat y \mid x)$ denote the pseudo-gradients of preactivations. Then, the K-FAC approximation of GNH for layer $l$ is given by:
\begin{align}\label{eq:kfac}
    \bm{G}_{\text{cap}}^{l} \approx \mathbb{E}\left[\bm{a}_{l-1}\bm{a}_{l-1}^\top \right] \otimes \mathbb{E}\left[ \bm{g}_l\bm{g}_l^\top\right]  \coloneqq \bm{A}_{l-1} \otimes \bm{S}_l.
\end{align}
Here, $\bm{A}_{l-1}$ and $\bm{S}_l$ are uncentered covariance matrices of the activations and preactivation pseudo-gradients, respectively, with the expectation taken with respect to the capabilities dataset $\mathcal{D}_{\text{cap}}$. This reduces the per-layer storage requirements from $O(d_{\text{in}}^2 d_{\text{out}}^2)$ to $O(d_{\text{in}}^2+d_{\text{out}}^2)$ memory.

\textbf{Matrix-free projections without forming $\bm{P}^{(l)}_\gamma$.}
Even with K-FAC approximations in place, explicitly materializing a projector matrix $\bm{P}^{(l)}_\gamma$ for the $\gamma$-approximate nullspace of $\bm{G}_{\text{cap}}^{(l)}$ is memory-prohibitive. Thus, we now describe an efficient method to project onto the $\gamma$-approximate nullspace that does not require explicitly forming $\bm{P}^{(l)}_\gamma$. The key idea behind our approach is the fact that the eigenvalues/eigenvectors of a Kronecker product $\bm{M} \otimes \bm{N}$ are simply the product of the eigenvalues/eigenvectors of $\bm{M}$ and $\bm{N}$. Specifically, let $\bm{A}_{l-1} = \bm{U}_{\text{in}}\bm{\Lambda}_{\text{in}}\bm{U}_{\text{in}}^\top$ and $\bm{S}_{l-1} = \bm{U}_{\text{out}}\bm{\Lambda}_{\text{out}}\bm{U}_{\text{out}}^\top$ denote the respective eigendecompositions of $\bm{A}_{l-1}$ and $\bm{S}_{l-1}$. We show in Appendix \ref{sec:matrix_projection_proof}, for a weight-gradient $\bm{Q}_l = \nabla_{\bm{W}_l} L_{\text{edit}}(\bm{\theta})$, the projected gradient $\bm{Q}^{\text{proj}}_l = \MatOp(\bm{P}^{(l)}_\gamma \VecOp( \bm{Q}_l ))$ can be written as:
\begin{align}\label{eq:matrix_free_proj}
\bm{Q}^{\text{proj}}_l
\; & =\;
\bm{U}_{\text{out}}\Big(\big(\bm{U}_{\text{out}}^\top \bm{Q}_l \bm{U}_{\text{in}}\big)\odot \bm{M}\Big)\bm{U}_{\text{in}}^\top,
\end{align}
where $\odot$ denotes the Hadamard (entry-wise) matrix product and $\bm{M}_{ij} = \bm{1}\!\left[\lambda^{\text{out}}_i\lambda^{\text{in}}_j \le \lambda_\gamma\right]$ is a binary mask that selects low-curvature components of the Kronecker matrix; $\lambda_\gamma$ denotes the largest eigenvalue associated with the $\gamma$-approximate nullspace of $\bm{P}_\gamma^\ell$. Using this formula, one never needs to form the $d_{\text{in}}d_{\text{out}}\times d_{\text{in}}d_{\text{out}}$ projector, further reducing the storage requirement from $O(d_{\text{in}}^2 d_{\text{out}}^2)$ to $O(d_{\text{in}}^2 + d_{\text{out}}^2 + d_{\text{in}} d_{\text{out}})$. With this projection in hand, we are ready to define \ourmethod, presented in~\cref{alg:crispedit_batch}.
\begin{algorithm}[t]
\caption{\ourmethod}
\label{alg:crispedit_batch}
\begin{algorithmic}[1]
\REQUIRE $\bm{\theta}_0$, $\mathcal{D}_{\text{cap}}$, $\mathcal{D}_{\text{edit}}$, number of epochs $E$.
\ENSURE Edited model parameters $\bm{\theta}$.
\STATE Compute K-FAC factors $(\bm{A}_{l-1}, \bm{S}_l)$ for all finetuned layers $l$ on $\cL(\bm{\theta}; \mathcal{D}_{\text{cap}})$;
cache $\bm{U}_{\text{out}}^{(l)}, \bm{U}_{\text{in}}^{(l)}$, and projection mask $\bm{M}^{(l)}$ for each layer (computed via SVD).
\STATE Initialize parameters $\bm{\theta} \leftarrow \bm{\theta}_0$.
\FOR{$e = 1$ to $E$}
    \FOR{each minibatch $\mathcal{B} \subset \mathcal{D}_{\text{edit}}$}
        \STATE Compute gradient $\bm{Q}_l$ for each fine-tuned layer $l$.
        \STATE Project gradient to $\bm{Q}_l^{\text{proj}}$ (cf.~\cref{eq:matrix_free_proj})
        \STATE Update parameters $\bm{\theta}$ using $\bm{Q}_l^{\text{proj}}$.
    \ENDFOR
\ENDFOR
\end{algorithmic}
\end{algorithm}

\begin{greenAIbox}{Result: K-FAC enables scalable matrix-free curvature projection.}
Directly storing the GNH or its projector is memory-prohibitive. We overcome this by approximating the GNH with K-FAC and deriving a \textit{matrix-free} projection update. By exploiting the Kronecker structure, we project gradients using only the eigendecompositions of the smaller factor matrices $\bm{A}$ and $\bm{S}$, avoiding the materialization of the full high-dimensional projector. This reduces memory complexity from $O(d_{\text{in}}^2 d_{\text{out}}^2)$ to $O(d_{\text{in}}^2 + d_{\text{out}}^2)$, enabling \ourmethod to scale efficiently.
\end{greenAIbox}

\begin{algorithm}[b]
\caption{\ourmethodseq}
\label{alg:online_KFAC_updates}
\begin{algorithmic}[1]
\REQUIRE $\bm{\theta}_0$, $\mathcal{D}_{\text{cap}}$, edits $\mathcal{D}_{\text{edit}}^{(1)}, \dots, \mathcal{D}_{\text{edit}}^{(K)}$.
\ENSURE Edited models $\bm{\theta}_1, \dots, \bm{\theta}_K$ (updated sequentially).

\STATE Compute K-FAC factors $\{\bm{A}^{(l-1)}, \bm{S}^{(l)}\}$ on $\cL(\bm{\theta}; \mathcal{D}_{\text{cap}})$.
\STATE Initialize
$\{\bm{A}^{(l-1)}_{\text{acc}}, \bm{S}^{(l)}_{\text{acc}}\}
\leftarrow
\{\bm{A}^{(l-1)}_{\text{cap}}, \bm{S}^{(l)}_{\text{cap}}\}$.

\FOR{$k = 1$ to $K$}
    \STATE Solve \eqref{eq:batch_constrained_optimization} for $\bm{\theta}_k$ with edit loss $\cL(\bm{\theta}; \mathcal{D}_{\text{edit}}^{(k)})$, using layer-wise $\gamma$-approximate nullspace projections induced by
    $\{\bm{A}^{(l-1)}_{\text{acc}}, \bm{S}^{(l)}_{\text{acc}}\}$ (cf.~\Cref{alg:crispedit_batch}).
    \STATE Compute K-FAC factors $\{\bm{A}^{(l-1)}_{\text{edit},k}, \bm{S}^{(l)}_{\text{edit},k}\}$ for $\mathcal{D}_{\text{edit}}^{(k)}$.
    \STATE Aggregate K-FAC factors $\{\bm{A}^{(l-1)}_{\text{acc}}, \bm{S}^{(l)}_{\text{acc}}\}$ with $\{\bm{A}^{(l-1)}_{\text{edit},k}, \bm{S}^{(l)}_{\text{edit},k}\}$ via streaming averages.
\ENDFOR
\end{algorithmic}
\end{algorithm}

\subsection{Sequential editing via online projection updates}
\label{sec:sequential_updates}
Up to this point, we have described \ourmethod in a \emph{batch editing setting}, where we assume all the edits $\mathcal{D}_{\text{edit}}$ are gathered at once, and the base model is updated to incorporate all the edits. A complementary setting is one of \textit{sequential editing}, where edits (single instances or batches) arrive over time and the model is updated from $f_{\bm{\theta}_0}$ to $f_{\bm{\theta}_1}, \dots, f_{\bm{\theta}_K}$ in $K$ successive rounds. Here, at every round $k$, the goal is to preserve both the base capabilities and the earlier edits in rounds $1$ to $k-1$ applied to the model. This setting is closely connected to continual (or lifelong) learning \cite{de2021continual,shi2025continual} and inherits its core failure mode catastrophic forgetting. Batch editing can be viewed as ``breadth-first'', integrating many edits at once, whereas sequential editing is ``depth-first'', repeatedly revising the model as the new edit data arrive \cite{yang2025finetuning}.

Concretely, consider a sequence of edit data that arrives over time in chunks: $\mathcal{D}_{\text{edit}}^{(1)}, \dots, \mathcal{D}_{\text{edit}}^{(K)}$. A na{\"{i}}ve algorithm at every round $k$ sets $\mathcal{D}_{\text{edit}} = \cup_{i=1}^{k} \mathcal{D}_{\text{edit}}^{(i)}$, and approximately solves problem \eqref{eq:batch_constrained_optimization} using \Cref{alg:crispedit_batch}. However, this na{\"{i}}ve approach must keep all edits around, which can be infeasible and/or impractical for large $K$ or privacy-sensitive settings \cite{yao2023editing}. To address these issues, we develop an algorithm (\Cref{alg:online_KFAC_updates}) which sequentially maintains the requisite statistics to implement $\gamma$-approximate nullspace projection. The key idea behind \Cref{alg:online_KFAC_updates} is that the $\bm{A}_{l-1}$ and $\bm{S}_l$ factors from K-FAC (cf.~\eqref{eq:kfac}) are memory-efficient sufficient statistics to summarize the approximate nullspace of the capability loss and the previous edit losses. By updating these statistics online after each round $k$, we can simultaneously minimize $\cL(\bm{\theta}; \mathcal{D}_{\text{edit}}^{(k)})$ while treating both capabilities and the existing edit losses as hard constraints.

\section{Experiments}
\subsection{Comparison of various second-order constraints}
\label{sec:hessian_approximations}

To understand the effect of various second-order constraints on 
capability preservation in model editing, we consider a simple setting where calculating the Hessian of the model is tractable. Since this is prohibitive for large LLMs, we use LeNet-5~\citep{lecun2002gradient} as a representative model. We pre-train the model to 99\% test accuracy on the MNIST dataset~\citep{lecun1998mnist} and fine-tune it on the Fashion-MNIST dataset~\citep{xiao2017fashion}. In this setting, we treat the MNIST loss as the capabilities objective, and the Fashion-MNIST loss as the edit objective.
\begin{wrapfigure}[24]{r}{0.45\textwidth}
    \centering
    \includegraphics[width=\linewidth]{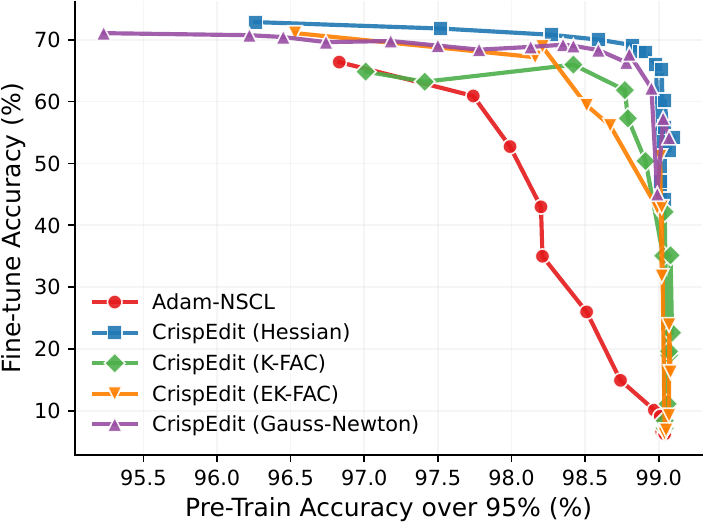}
    \caption{\textbf{Tradeoff between pre-training accuracy (capability preservation) and post-training performance (edit efficacy) for different nullspace projection methods.} We fine-tune a LeNet-5 model pre-trained on MNIST on Fashion-MNIST in the $\gamma$-approximate nullspace of the embeddings (Adam-NSCL) Hessian along with Hessian approximations Gauss-Newton Hessian, K-FAC, and EK-FAC (\ourmethod), over a range of energy thresholds $\gamma$.
    }
    \label{fig:roc_curve}
\end{wrapfigure}

For the fine-tuning phase, we first compute the $\gamma$-approximate nullspace projector of the Hessian of the pre-train loss, applying projected gradient descent (PGD) to fine-tune a one hidden-layer MLP, as described in \Cref{sec:method_hessian}. To address the inaccuracy of the projector caused by parameter drift, we recalculate the $\gamma$-approximate nullspace projector every time parameter changes more than 25\%. To understand the trade-off curve between pre-train and fine-tune test accuracy, we sweep over a range of energy threshold $\gamma = 1-10^{-k}$ with $k\in [\frac{1}{10}, 7]$. We then compare this algorithm against running PGD onto four alternative approximate nullspaces:
(a) activation covariance (cf.~Adam-NSCL~\cite{wang2021training}),
(b) Gauss-Newton Hessian,
(c) K-FAC~\cite{martens2015optimizing}, and
(d) eigenvalue-corrected K-FAC (EK-FAC)~\cite{george2018fast}.

Our results, which illustrate the trade-off between pre-train and fine-tune performance for both the Hessian-based algorithm and the four alternatives (a)-(d), are shown in \cref{fig:roc_curve}. We highlight three findings: (i) Projecting gradient updates onto the $\gamma$-approximate nullspace of the Hessian provides an effective strategy for improving fine-tune accuracy on Fashion-MNIST while maintaining base MNIST performance. (ii) The GNH approach yields a trade-off curve that is quite competitive with the Hessian approach, illustrating the efficacy of the Bregman constraint. This, however, is not the case with the activation covariance used by Adam-NSCL. (iii) Both K-FAC and EK-FAC approximate the performance of the GNH approach reasonably well. The last point (iii) is promising, as it suggests that using K-FAC when we are unable to compute the full Hessian (e.g., LLMs) is a viable approach as we demonstrate next.

\begin{table}[!t]
\caption{\textbf{Comparison of \ourmethod with existing methods on editing \llama.} \textit{Rel} and \textit{Gen} denote reliability and generalization. We edit 3{,}000 samples from three datasets, evaluate edits with WILD, and measure base capability on five benchmarks. Values that are best or within 5\% of best are in bold.}
\centering
\resizebox{\textwidth}{!}{%
\begin{tabular}{llcccccccccc}
\toprule
\multirow{3}{*}{\textbf{Data}} & \multirow{3}{*}{\textbf{Method}} & \multicolumn{4}{c}{\textbf{Edited Capabilities}} & \multicolumn{5}{c}{\textbf{Base Capabilities}} & \multirow{3}{*}{\textbf{Time}} \\
\cmidrule(lr){3-6} \cmidrule(lr){7-11}
 & & \multicolumn{2}{c}{\textbf{QA Context}} & \multicolumn{2}{c}{\textbf{No Context}} & & & & & & \\
\cmidrule(lr){3-4} \cmidrule(lr){5-6}
 & & \textbf{Rel} & \textbf{Gen} & \textbf{Rel} & \textbf{Gen} & \textbf{MMLU} & \textbf{IFEval} & \textbf{TruthfulQA} & \textbf{ARC-C} & \textbf{GSM8K} & \\
\midrule
% --- ZsRE ---
\multirow{13}{*}{\rotatebox[origin=c]{90}{\textbf{ZsRE}}}
 & \llama{} & 2.1 & 1.7 & 2.9 & 2.1 & 69.5 & 69.3 & 50.7 & 58.0 & 73.5 & {  } \\
\cline{2-12}

 & MEMIT & 0.1 & 0.0 & 0.1 & 0.1 & 22.9 & 0.0 & 51.3 & 23.5 & 0.0 & {\small 9h 27m} \\
 & AlphaEdit & 70.1 & 60.6 & 48.1 & 39.4 & 52.7 & 47.7 & 46.3 & 40.5 & 45.5 & {\small 7h 19m} \\
 & Adam-NSCL & 16.6 & 15.5 & 1.9 & 2.0 & \textbf{69.2} & 29.6 & 50.8 & 42.0 & 39.5 & {\small 29m 19s} \\
 & LocBF-FT & 69.5 & 59.7 & 25.2 & 22.1 & \textbf{69.5} & \textbf{70.1} & 51.6 & \textbf{54.0} & \textbf{75.5} & {\small 22m 15s} \\
 & UltraEdit & 20.0 & 16.3 & 22.7 & 17.4 & \textbf{69.3} & \textbf{72.5} & 51.8 & \textbf{54.5} & \textbf{73.0} & {\small 3m 23s} \\
 & MEND & 0.0 & 0.0 & 0.0 & 0.0 & 22.9 & 18.2 & 0.0 & 26.0 & 0.0 & {\small 58m 20s} \\
 & FT & 46.8 & 43.1 & 9.9 & 8.3 & \textbf{69.3} & 45.0 & 48.7 & 43.0 & 50.0 & {\small 4m 32s} \\
 & FT Sequential & 3.6 & 3.5 & 0.9 & 1.2 & \textbf{68.8} & 19.4 & 52.8 & 40.5 & 6.5 & {\small 9m 17s} \\
 & LoRA & 9.1 & 7.4 & 18.7 & 7.2 & \textbf{67.8} & \textbf{70.8} & 52.0 & \textbf{56.0} & 71.0 & {\small 47m 24s} \\
 & LoRA Sequential & 4.4 & 4.0 & 1.3 & 0.9 & \textbf{67.3} & 64.6 & \textbf{56.0} & 47.0 & 67.0 & {\small 3h 12m} \\
\rowcolor{jigsawGreen!20}
 & \ourmethod{} & \textbf{80.5} & \textbf{69.0} & 57.4 & 50.9 & \textbf{69.5} & 67.9 & 50.5 & \textbf{55.0} & \textbf{76.0} & {\small 4m 6s} \\
\rowcolor{jigsawGreen!20}
 & \ourmethodseq{} & 71.1 & 62.9 & \textbf{72.8} & \textbf{60.6} & \textbf{67.8} & \textbf{70.2} & \textbf{53.6} & 52.0 & \textbf{74.0} & {\small 43m 36s} \\
\midrule
% --- CounterFact ---
\multirow{13}{*}{\rotatebox[origin=c]{90}{\textbf{CounterFact}}}
 & \llama{} & 1.2 & 1.0 & 0.3 & 0.6 & 69.5 & 69.3 & 50.7 & 58.0 & 73.5 & {   }\\
\cline{2-12}
 & MEMIT & 0.0 & 0.0 & 0.0 & 0.0 & 24.6 & 18.6 & 49.6 & 21.0 & 0.0 & {\small 7h 30m} \\
 & AlphaEdit & 74.9 & \textbf{57.0} & \textbf{50.5} & \textbf{44.1} & 47.4 & 32.9 & 41.5 & 40.5 & 37.5 & {\small 5h 56m} \\
 & Adam-NSCL & 19.1 & 8.5 & 1.7 & 1.8 & \textbf{68.6} & 22.8 & \textbf{57.1} & 39.5 & 16.5 & {\small 24m 9s} \\
 & LocBF-FT & 61.1 & 41.6 & 10.9 & 13.3 & \textbf{69.4} & 65.0 & 51.3 & \textbf{52.5} & \textbf{74.0} & {\small 14m 40s} \\
 & UltraEdit & 18.1 & 12.4 & 10.2 & 9.3 & \textbf{69.2} & \textbf{68.6} & 49.2 & \textbf{52.0} & \textbf{74.0} & {\small 3m 9s} \\
 & MEND & 0.0 & 0.0 & 0.0 & 0.0 & 22.9 & 18.2 & 0.0 & 26.0 & 0.0 & {\small 17m 42s} \\
 & FT & 12.3 & 6.0 & 1.6 & 2.2 & \textbf{67.4} & 22.7 & 50.4 & 40.0 & 18.0 & {\small 4m 12s} \\
 & FT Sequential & 19.1 & 10.6 & 1.3 & 2.2 & 33.4 & 20.4 & 51.3 & 31.5 & 0.0 & {\small 6m 45s} \\
 & LoRA & 13.2 & 8.3 & 9.5 & 2.7 & \textbf{68.2} & \textbf{68.8} & 53.4 & \textbf{53.0} & 71.0 & {\small 51m 34s} \\
 & LoRA Sequential & 6.5 & 4.8 & 1.6 & 2.0 & \textbf{67.3} & 62.4 & 53.9 & 40.0 & 71.0 & {\small 2h 16m} \\
\rowcolor{jigsawGreen!20}
 & \ourmethod & \textbf{79.4} & \textbf{55.9} & 38.4 & 32.4 & \textbf{69.3} & \textbf{67.5} & 49.5 & \textbf{54.0} & \textbf{76.5} & {\small 3m 17s} \\
\rowcolor{jigsawGreen!20}
 & \ourmethodseq{} & 66.5 & 43.8 & 39.1 & 29.2 & \textbf{67.9} & \textbf{68.5} & \textbf{56.6} & \textbf{54.0} & \textbf{73.0} & {\small 34m 39s} \\
\midrule
% --- WikiBigEdit ---
\multirow{13}{*}{\rotatebox[origin=c]{90}{\textbf{WikiBigEdit}}}
 & \llama{} & 9.3 & 9.1 & 16.4 & 16.1 & 69.5 & 69.3 & 50.7 & 58.0 & 73.5 & {  } \\
\cline{2-12}
 & MEMIT & 0.0 & 0.0 & 0.0 & 0.0 & 24.6 & 13.6 & 52.3 & 23.5 & 0.0 & {\small 10h 42m} \\
 & AlphaEdit & 72.9 & \textbf{66.8} & \textbf{73.9} & \textbf{68.3} & 58.5 & 61.6 & 50.2 & 50.5 & 58.0 & {\small 7h 37m} \\
 & Adam-NSCL & 13.6 & 13.6 & 3.4 & 3.4 & \textbf{69.2} & 45.3 & 50.2 & 42.5 & 39.0 & {\small 30m 45s} \\
 & LocBF-FT & 50.4 & 46.7 & 16.7 & 15.7 & \textbf{69.2} & \textbf{73.2} & 52.0 & \textbf{55.5} & \textbf{73.5} & {\small 15m 47s} \\
 & UltraEdit & 59.2 & 54.8 & 55.4 & 52.0 & \textbf{69.3} & 67.7 & 52.4 & \textbf{53.5} & \textbf{74.5} & {\small 3m 15s} \\
 & MEND & 0.0 & 0.0 & 0.0 & 0.0 & 22.9 & 18.2 & 0.0 & 26.0 & 0.0 & {\small 38m 36s} \\
 & FT & 23.3 & 23.2 & 4.2 & 4.3 & \textbf{69.5} & 49.4 & 49.2 & 42.5 & 59.0 & {\small 5m 12s} \\
 & FT Sequential & 13.4 & 12.6 & 1.8 & 1.5 & \textbf{68.1} & 34.5 & 51.8 & 43.0 & 29.5 & {\small 10m 13s} \\
 & LoRA & 30.0 & 25.8 & 27.0 & 15.7 & \textbf{67.8} & \textbf{70.7} & \textbf{55.4} & 48.0 & \textbf{75.0} & {\small 58m 42s} \\
 & LoRA Sequential & 20.9 & 18.7 & 7.9 & 7.3 & \textbf{67.8} & \textbf{73.8} & \textbf{54.4} & 48.0 & 71.0 & {\small 4h 54m} \\
\rowcolor{jigsawGreen!20}
 & \ourmethod{} & \textbf{77.0} & \textbf{70.2} & 28.4 & 30.5 & \textbf{69.3} & \textbf{70.5} & 51.8 & \textbf{55.0} & \textbf{74.0} & {\small 6m 29s} \\
\rowcolor{jigsawGreen!20}
 & \ourmethodseq{} & 66.7 & 59.8 & 40.8 & 38.6 & \textbf{69.2} & 68.8 & 50.4 & \textbf{53.0} & \textbf{73.0} & {\small 38m 47s} \\
\bottomrule
\end{tabular}
}%  end resizebox
\label{tab:main_table}
\end{table}

\subsection{Large-scale LLM evaluations}\label{sec:llm_experiments}
We now study scaling \ourmethod to billion-parameter LLMs, predominately focusing on \llama. We investigate the following: (i) How well can we edit the model? (ii) Do the edits generalize for different contexts? (iii) To what extent can we preserve the model capabilities? 

\noindent\textbf{Datasets, metrics, and evaluation.}
We edit the base model on 3000 samples of three standard model editing datasets: ZsRE~\cite{levy-etal-2017-zero}, CounterFact~\cite{meng2022locating}, and WikiBigEdit~\cite{thede25awikibigedit}. We report two standard edit metrics~\citep{de2021editing, yang-etal-2025-mirage}: \emph{reliability} (or efficacy) asks whether the edited model produces an acceptable answer to a given edit query, and \emph{generalization} asks whether the effects of an edit extend to semantically related contexts. All three datasets contain rewrite prompts for efficacy evaluation, and paraphrased prompts for generalization evaluation. To measure capability degradation, we benchmark edited and base models on diverse tasks: MMLU~\citep{hendrycks2020measuring}, IFEval~\citep{zhou2023instruction}, TruthfulQA~\citep{lin2022truthfulqa}, ARC-Challenge~\citep{clark2018think}, and GSM8k~\citep{cobbe2021training}.

An edited LM should apply the edits in a \textit{conversational} manner and across different contexts. Yet, due to the computational costs, prior works~\citep{fang2025alphaedit, gu2025ultraedit} typically use likelihood-based, teacher-forced evaluation that leak both content and length of the ground truth, leading to overestimated performance~\citep{yang-etal-2025-mirage}. To better capture realistic editing behavior, we follow the WILD evaluation protocol~\citep{yang-etal-2025-mirage} that combines context-guided autoregressive decoding of LLM responses with LLM-as-a-judge evaluation. We adopt WILD with EasyEdit~\citep{wang2024easyedit}, evaluating prompts both with and without QA context. While we do not anticipate any real-world carry-over, we include teacher-forced evaluations in~\cref{tab:teacher_forced_benchmark} (Appendix) for 
completeness.

\noindent\textbf{Method and baselines.}
We edit the base model with \ourmethod by first computing K-FAC caches on Wikipedia samples for five MLP down-projection layers, and then fine-tuning them with PGD in the $\gamma$-approximate nullspace of caches (cf. \cref{alg:crispedit_batch}). We compare against a range of baselines. % chosen for their merit and relevance. 
MEMIT~\citep{meng2023massediting} and AlphaEdit~\citep{fang2025alphaedit} follow the locate-then-edit paradigm; Adam-NSCL~\citep{wang2021training} performs PGD in the feature covariance nullspace; UltraEdit~\citep{gu2025ultraedit} leverages sensitivity analysis with online statistics; MEND~\citep{mitchell2022fast} uses a hypernetwork to predict parameter changes, FT and LoRA~\citep{hu2022lora, zhang2023adaptive} performs standard and low-rank fine-tuning, respectively; and LocBF-FT~\citep{yang2025finetuning} constrains fine-tuning to a single, hyperparameter-tuned layer. For more details about the evaluation and baselines, see Appendix~\ref{sec:llm_experiment_details}.

\begin{wrapfigure}[19]{r}{0.55\textwidth}
    \centering
    \includegraphics[width=\linewidth]{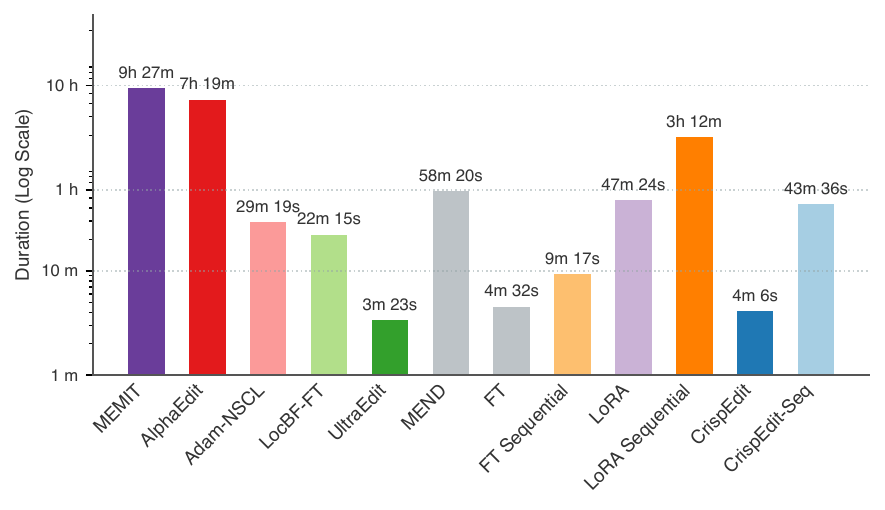}
    \caption{\textbf{Runtime comparison of \ourmethod with other methods.} We apply a number of model editing methods to edit \llama on 3,000 ZsRE samples and measure the wall-clock time for execution.
    }
    \label{fig:duration_bar_plot}
\end{wrapfigure}
\noindent\textbf{Key results.}
We report our key results in \cref{tab:main_table}. Across all datasets, we find two consistent patterns. First, aggressive editing approaches---including MEMIT, MEND, FT, and Adam-NSCL---exhibit substantial degradation. While these methods perform well under teacher-forced evaluation (cf.~\cref{tab:teacher_forced_benchmark}, Appendix), the degraded base capabilities adversely affect their editing performance under autoregressive decoding (cf. Appendix~\ref{sec:qualitative_results}). Second, conservative editing strategies, which restrict updates to limited parameter subspaces, better preserve base capabilities but lead to a suboptimal edited capabilities. AlphaEdit remains a strong baseline of this class, yet it degrades the model's base capabilities due to its limited nullspace estimate, in addition to needing additional subject-centric representations. In comparison, \textbf{\ourmethod consistently tops editing performance while preserving the base capabilities nearly intact.} Furthermore, it remains computationally efficient (cf.~\Cref{fig:duration_bar_plot}), as it only augments standard fine-tuning with PGD. 

\begin{greenAIbox}{Result: \methodName\xspace achieves superior editing performance while preserving base model capabilities.}
Prior editing methods often trade capability preservation for edit quality. Approaches like FT and Adam-NSCL can lead to substantial degradation under autoregressive decoding, while conservative methods such as AlphaEdit require pushing the energy threshold so low that the resulting nullspace becomes a loose approximation, thus improving edits at the cost of base capabilities. \ourmethod consistently yields a better trade-off and achieves high edit performance with nearly intact base capabilities, all the while maintaining computational efficiency via projected gradient descent fine-tuning.
\end{greenAIbox}

\noindent\textbf{Ablations.} We now discuss key findings from ablation experiments; results are provided in Appendix~\ref{sec:additional_tables}.

\noindent\textit{(i) Robustness to energy threshold $\gamma$.} 
We vary the threshold $\gamma$ from 0.5 to 0.99. \cref{tab:ablation_gamma} shows that {\ourmethod's base capability preservation is reasonably robust to the threshold}, even with $\gamma$ as small as $0.5$. 

\noindent\textit{(ii) Sensitivity to the size of capability dataset $n$.} We vary $n = | \mathcal{D}_{\text{cap}}|$ from 10 to 100,000. Surprisingly, as \cref{tab:ablation_n} shows, \ourmethod stays robust across a range of dataset sizes, maintaining strong base capability even with as few as 100 samples. This suggests {\ourmethod requires only a small cache to be effective}. This raises a question: are capability dataset needed \textit{at all?} To validate the importance of capability dataset, we run standard finetuning with no projections (i.e., $n=0$). As~\cref{fig:plot_ablation_n} shows, while \ourmethod is robust to $n$, lack of projection yields a detrimental effect on capability preservation. {Furthermore, capability preservation can improve edit performance in autoregressive evaluation through maintaining fluency and reliable instruction-following during generation.}
\begin{figure}[t]
    \centering
    \includegraphics[width=0.8\linewidth]{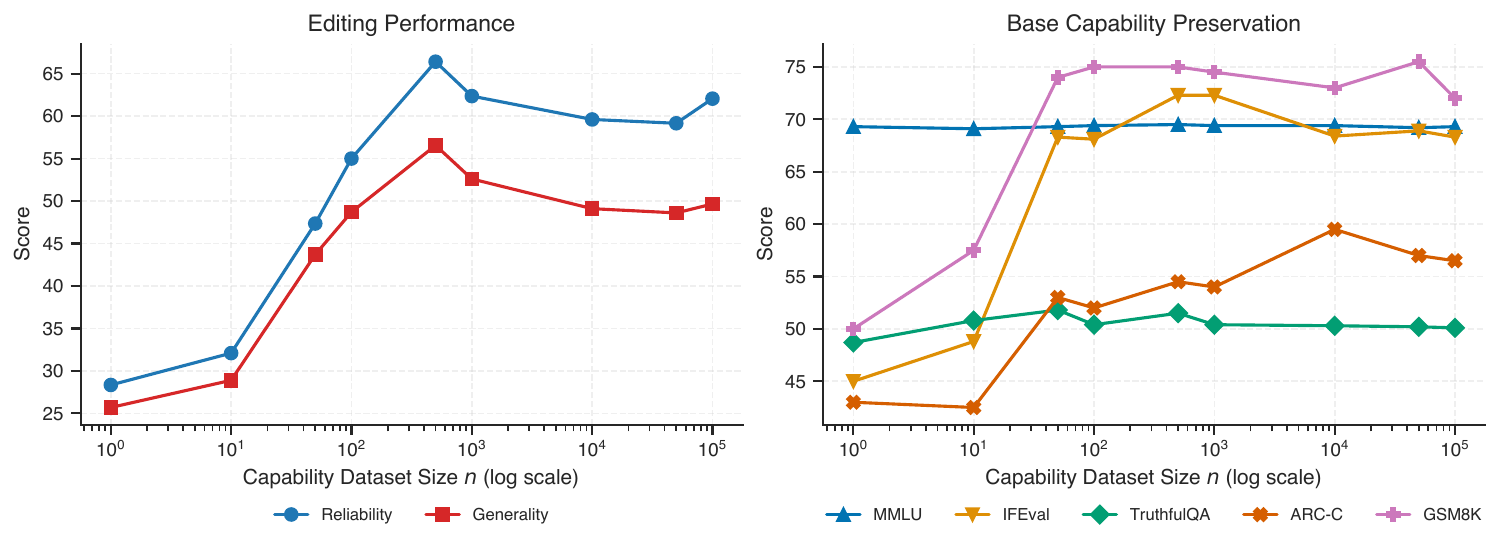}
    \caption{
    \textbf{Effect of capability dataset size $n$ on editing performance and base capability preservation.} We edit \llama on 3,000 ZsRE samples using \ourmethod for a range of $n$ and measure the editing performance and base capability preservation.}
    \label{fig:plot_ablation_n}
\end{figure}

\noindent\textit{(iii) Scaling to larger datasets.} We increase the size of the the edit dataset, using up to 10,000 ZsRE samples. As \cref{tab:benchmark_10k} shows, {\ourmethod scales robustly from 3,000 to 10,000 edits}. In contrast, the baselines degrade or plateau at larger scales due to sequential editing, restrictive layer choices, or limited adaptation capacity. Notably, while LocBF-FT performs competitively at 3k edits, its performance drops significantly at 10k edits. This degradation stems from its restriction to a single layer, which lacks the representational capacity required to manage larger-scale knowledge updates.

\noindent\textit{(iv) Sensitivity to model families.} 
We use \ourmethod to apply 3,000 ZsRE edits to \qwen, and compare it against strong baselines. As \cref{tab:benchmark_qwen} shows, our method retains its advantages, achieving strong editing performances while retaining base capabilities.

\begin{AIbox}{Takeaway: \methodName\xspace is robust to hyperparameter choices and scales to large-scale editing.}
Our ablations demonstrate that \ourmethod is highly resilient to variations in the energy threshold $\gamma$ and remains effective with a minimal capability cache (as few as 100 samples), though the projection mechanism itself remains essential. Unlike baselines that face capacity bottlenecks at scale (e.g., LocBF-FT), \ourmethod maintains performance up to 10,000 edits and generalizes effectively across different model architectures like \qwen.
\end{AIbox}
\begin{figure}[b]
    \centering
    % Left Figure
    \begin{minipage}{0.66\textwidth}
        \centering
        \includegraphics[width=\textwidth]{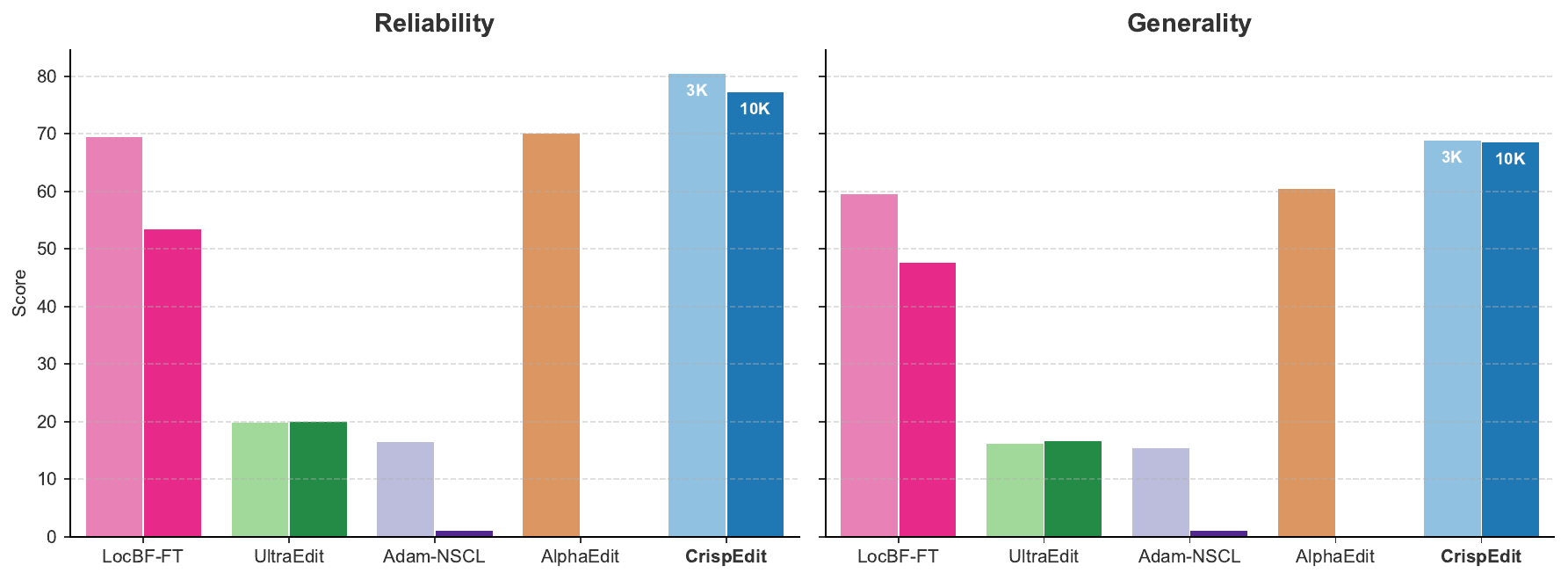}
        \caption{\textbf{Consequence of scaling the number of edits up to 10,000.} We edit \llama on 3,000 and 10,000 ZsRE samples using several model editing methods and measure their reliability and generality with QA context. Here, darker hue corresponds to larger editing samples.}
        \label{fig:ds_size_ablation_plot}
    \end{minipage}
    \hfill
    % Right Figure
    \begin{minipage}{0.3\textwidth}
        \centering
        \includegraphics[width=\textwidth]{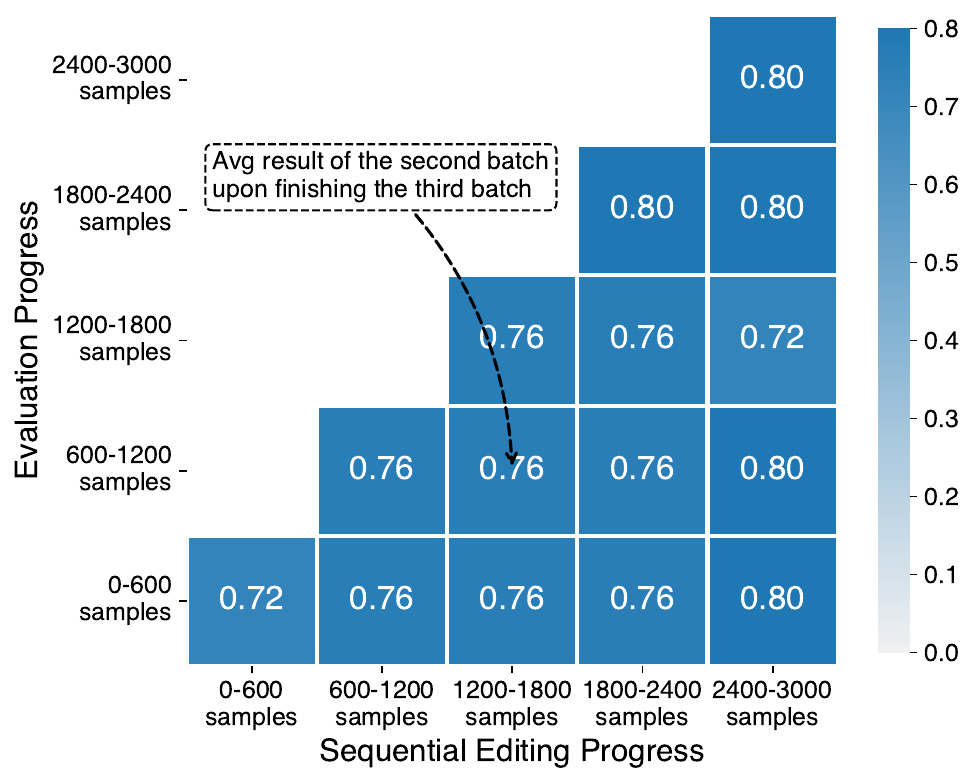}
        \caption{\textbf{Evolution of \ourmethodseq performance.} \ourmethodseq shows stronger editing performance whilst retaining previous edits.}
        \label{fig:chunk_eval}
    \end{minipage}
\end{figure}
% \newpage
% \input{figures/chunk_eval}
\noindent\textbf{Sequential editing with \ourmethodseq{}.}
\cref{tab:main_table} shows that \ourmethodseq{} matches the strength of \ourmethod in sequential editing. \ourmethodseq{} also reasonably matches the sequential editing performance of AlphaEdit (the strongest competitor), while retaining base capabilities nearly intact and operating $8\times$ faster. \Cref{fig:chunk_eval} shows that \ourmethodseq{} retains previously edited knowledge despite being a \textit{depth-first} fine-tuning method, challenging previous assumptions that depth-first methods are ill-suited for sequential model editing~\citep{yang2025finetuning}.

\section{Related work}
\label{sec:related_work}

\textbf{Memory-based approaches} employ additional memory components to store edits outside its parameters. These components can be in the form of axillary models \citep{dong2022calibrating, mitchell2022memory, hartvigsen2023aging, wang2024wise}, in-context learning \citep[WISE]{wang2024wise}, low-rank adapters \citep[MELO]{yu2024melo}, or retrieval-based alignment \citep[LTE]{jiang2024learning}. Compared to these methods, \ourmethod does not assume any data, memory, or architectural augmentations for inference. 

\textbf{Locate-then-edit based approaches} aim to locate a set of parameters responsible for a undesired behavior and edit them. They rely on the assumption that feed-forward networks contain the knowledge in models~\citep{geva2021transformer, geva2022transformer, dai2022knowledge} and precisely edit the neurons responsible for particular information. They often assume structures in the dataset such as subject or entity~\citep{meng2022locating, meng2023massediting, gupta2024unified, fang2025alphaedit, pan2025precise} and relations~\citep{dai2022knowledge}. An exception to these is \citet[UltraEdit]{gu2025ultraedit}, which uses representations of the last token for its localization calculation. In contrast, \ourmethod does not assume any edit structure and does not require locating specific parameters.

\textbf{Hypernet-based approaches} treat predicting parameters shifts as a meta-learning problem and learns a separate network to solve the problem. These methods take the underlying optimization problem of locate-then-edit methods and uses an hypernetwork to predict the parameter shifts, such as \citet[MEND]{mitchell2022fast} solving the optimization speed of \citet[ROME]{meng2022locating} and \citet[MALMEN]{tan2024massive} solving the least squares problem of \citet[MEMIT]{meng2023massediting}. Recently, \citet[RLEdit]{li2025reinforced} treats the dual optimization problem of model stability and edit quality by treating the hypernetwork as a reinforcement learning (RL) agent. Compared to these methods, \ourmethod has no additional network for predicting parameters shifts.

\textbf{Constrained fine-tuning approaches} perform GD-based finetuning with additional constraints such as weight decay~\citep[FT-L]{rawat2021modifying}, null-space projection~\citep[Adam-NSCL]{wang2021training}, prompt-masking~\citep[FT-M]{zhang2024comprehensive}, low-rank update~\citep[MELO]{yu2024melo} or strict layer choice~\citep[LocBF-FT]{yang2025finetuning}. \ourmethod builds on this line by combining FT-M with PGD, deriving the projection from a  constrained-optimization view of capability preservation leveraging the loss curvature. In this way, \ourmethod aims to reduce the amount of manual strictness (e.g., highly restrictive layer choices or aggressive update limitations) sometimes required for constrained fine-tuning baselines, while retaining the simplicity and scalability of standard fine-tuning. Closest to our method is Adam-NSCL, which applies PGD in the null space of activation covariances. We show that Adam-NSCL is a special, more conservative case (\Cref{prop:alphaedit-special-case}) and \ourmethod empirically outperforms it.

\textbf{Continual learning} (CL) is closely related to model editing that studies sequential updates while mitigating catastrophic forgetting. Existing methods broadly fall into three categories: \textit{regularization-based methods} aim to preserve relevant parameters~\citep{zenke2017continual}, \textit{replay-based methods} aim to efficiently replay past memories during training~\citep{shin2017continual, rebuffi2017icarl}, and \textit{architecture-based methods} adjust model architecture on the fly~\citep{rusu2016progressive}. Relevant to our work are \textit{curvature aware} methods, most notably elastic weight consolidation~\citep[EWC]{kirkpatrick2017overcoming}, which estimates old task curvature with the Fisher and adds it as a penalty alongside standard loss to minimize  curvature change. Relatedly, \citet[HALRP]{li2024hessian} performs automatic rank selection with Hessian information of the loss w.r.t base weights and low rank perturbation on the weights to obtain task weights. Recently, \cite{gupta2024unified} unify different CL methods under a single Bregman divergence-based objective. In comparison, \ourmethod avoids per-step auxiliary loss calculation and scales to LLM editing. 

\section{Conclusion and future work}

We formulate model editing as a quadratically constrained optimization problem, introducing \ourmethod{} and its sequential variant as scalable approaches for editing billion-parameter LLMs while preserving capabilities. Our method leverages Gauss–Newton Hessian eigenspaces, induced by a Bregman divergence constraint, to identify low-curvature directions where the capabilities loss is nearly invariant. We use K-FAC to design efficient projection onto these nullspaces, making \ourmethod{} practical at LLM scale.

Our work opens up several exciting avenues for future work. The first direction is exploring the use of \ourmethod{} in other applications, such as safety (e.g., editing out harmful generation and/or hallucinations) and personalization (e.g., changing response style to suit user preferences). Another interesting direction is to utilize \ourmethod{} for learning interpretable models, e.g., training models to minimize some notion of model complexity such as weight sparsity, feature disentanglement, etc., subject to maintaining model capabilities. Finally, on the algorithmic side, alternative techniques for non-linear constrained optimization, such as trust-region and sequential quadratic programming methods, could enable \ourmethod{} to take larger, more aggressive fine-tuning steps leading to further improvements on edit capabilities while preserving base capabilities.

\bibliographystyle{assets/plainnat}
\bibliography{paper}

@article{george2018fast,
  title={Fast approximate natural gradient descent in a {K}ronecker factored eigenbasis},
  author={George, Thomas and Laurent, C{\'e}sar and Bouthillier, Xavier and Ballas, Nicolas and Vincent, Pascal},
  journal={Advances in neural information processing systems},
  volume={31},
  year={2018}
}

@inproceedings{martens2015optimizing,
  title={Optimizing neural networks with {K}ronecker-factored approximate curvature},
  author={Martens, James and Grosse, Roger},
  booktitle={International conference on machine learning},
  pages={2408--2417},
  year={2015},
  organization={PMLR}
}

@inproceedings{dong2022calibrating,
  title={Calibrating Factual Knowledge in Pretrained Language Models},
  author={Dong, Qingxiu and Dai, Damai and Song, Yifan and Xu, Jingjing and Sui, Zhifang and Li, Lei},
  booktitle={Findings of the Association for Computational Linguistics: EMNLP 2022},
  pages={5937--5947},
  year={2022}
}

@inproceedings{mitchell2022memory,
  title={Memory-based model editing at scale},
  author={Mitchell, Eric and Lin, Charles and Bosselut, Antoine and Manning, Christopher D and Finn, Chelsea},
  booktitle={International Conference on Machine Learning},
  pages={15817--15831},
  year={2022},
  organization={PMLR}
}

@article{hartvigsen2023aging,
  title={Aging with grace: {L}ifelong model editing with discrete key-value adaptors},
  author={Hartvigsen, Tom and Sankaranarayanan, Swami and Palangi, Hamid and Kim, Yoon and Ghassemi, Marzyeh},
  journal={Advances in Neural Information Processing Systems},
  volume={36},
  pages={47934--47959},
  year={2023}
}

@inproceedings{yu2024melo,
  title={{MELO: E}nhancing model editing with neuron-indexed dynamic {LoRA}},
  author={Yu, Lang and Chen, Qin and Zhou, Jie and He, Liang},
  booktitle={Proceedings of the AAAI Conference on Artificial Intelligence},
  volume={38},
  pages={19449--19457},
  year={2024}
}

@article{wang2024wise,
  title={{WISE}: {R}ethinking the knowledge memory for lifelong model editing of large language models},
  author={Wang, Peng and Li, Zexi and Zhang, Ningyu and Xu, Ziwen and Yao, Yunzhi and Jiang, Yong and Xie, Pengjun and Huang, Fei and Chen, Huajun},
  journal={Advances in Neural Information Processing Systems},
  volume={37},
  pages={53764--53797},
  year={2024}
}

@inproceedings{jiang2024learning,
  title={Learning to Edit: {A}ligning {LLM}s with Knowledge Editing},
  author={Jiang, Yuxin and Wang, Yufei and Wu, Chuhan and Zhong, Wanjun and Zeng, Xingshan and Gao, Jiahui and Li, Liangyou and Jiang, Xin and Shang, Lifeng and Tang, Ruiming and others},
  booktitle={Proceedings of the 62nd Annual Meeting of the Association for Computational Linguistics (Volume 1: Long Papers)},
  pages={4689--4705},
  year={2024}
}

@article{meng2022locating,
  title={Locating and editing factual associations in {GPT}},
  author={Meng, Kevin and Bau, David and Andonian, Alex and Belinkov, Yonatan},
  journal={Advances in neural information processing systems},
  volume={35},
  pages={17359--17372},
  year={2022}
}

@inproceedings{dai2022knowledge,
  title={Knowledge neurons in pretrained transformers},
  author={Dai, Damai and Dong, Li and Hao, Yaru and Sui, Zhifang and Chang, Baobao and Wei, Furu},
  booktitle={Proceedings of the 60th Annual Meeting of the Association for Computational Linguistics (Volume 1: Long Papers)},
  pages={8493--8502},
  year={2022}
}

@inproceedings{
meng2023massediting,
title={Mass-Editing Memory in a Transformer},
author={Kevin Meng and Arnab Sen Sharma and Alex J Andonian and Yonatan Belinkov and David Bau},
booktitle={The Eleventh International Conference on Learning Representations },
year={2023},
url={https://openreview.net/forum?id=MkbcAHIYgyS}
}

@inproceedings{gupta2024unified,
  title={A Unified Framework for Model Editing},
  author={Gupta, Akshat and Sajnani, Dev and Anumanchipalli, Gopala},
  booktitle={Findings of the Association for Computational Linguistics: EMNLP 2024},
  pages={15403--15418},
  year={2024}
}

@inproceedings{
fang2025alphaedit,
title={AlphaEdit: {N}ull-Space Constrained Model Editing for Language Models},
author={Junfeng Fang and Houcheng Jiang and Kun Wang and Yunshan Ma and Jie Shi and Xiang Wang and Xiangnan He and Tat-Seng Chua},
booktitle={The Thirteenth International Conference on Learning Representations},
year={2025},
url={https://openreview.net/forum?id=HvSytvg3Jh}
}

@inproceedings{
pan2025precise,
title={Precise Localization of Memories: {A} Fine-grained Neuron-level Knowledge Editing Technique for {LLM}s},
author={Haowen Pan and Xiaozhi Wang and Yixin Cao and Zenglin Shi and Xun Yang and Juanzi Li and Meng Wang},
booktitle={The Thirteenth International Conference on Learning Representations},
year={2025},
url={https://openreview.net/forum?id=5xP1HDvpXI}
}

@article{gu2025ultraedit,
  title={UltraEdit: Training-, Subject-, and Memory-Free Lifelong Editing in Large Language Models},
  author={Gu, Xiaojie and Chen, Guangxu and Li, Jungang and Gu, Jia-Chen and Hu, Xuming and Zhang, Kai},
  journal={arXiv preprint arXiv:2505.14679},
  year={2025}
}

@inproceedings{geva2021transformer,
  title={Transformer feed-forward layers are key-value memories},
  author={Geva, Mor and Schuster, Roei and Berant, Jonathan and Levy, Omer},
  booktitle={Proceedings of the 2021 Conference on Empirical Methods in Natural Language Processing},
  pages={5484--5495},
  year={2021}
}

@inproceedings{geva2022transformer,
  title={Transformer feed-forward layers build predictions by promoting concepts in the vocabulary space},
  author={Geva, Mor and Caciularu, Avi and Wang, Kevin and Goldberg, Yoav},
  booktitle={Proceedings of the 2022 conference on empirical methods in natural language processing},
  pages={30--45},
  year={2022}
}

@inproceedings{
mitchell2022fast,
title={Fast Model Editing at Scale},
author={Eric Mitchell and Charles Lin and Antoine Bosselut and Chelsea Finn and Christopher D Manning},
booktitle={International Conference on Learning Representations},
year={2022},
url={https://openreview.net/forum?id=0DcZxeWfOPt}
}

@inproceedings{
tan2024massive,
title={Massive Editing for Large Language Models via Meta Learning},
author={Chenmien Tan and Ge Zhang and Jie Fu},
booktitle={The Twelfth International Conference on Learning Representations},
year={2024},
url={https://openreview.net/forum?id=L6L1CJQ2PE}
}

@inproceedings{
li2025reinforced,
title={Reinforced Lifelong Editing for Language Models},
author={Zherui Li and Houcheng Jiang and Hao Chen and Baolong Bi and Zhenhong Zhou and Fei Sun and Junfeng Fang and Xiang Wang},
booktitle={Forty-second International Conference on Machine Learning},
year={2025},
url={https://openreview.net/forum?id=1jUXprrfcb}
}

@inproceedings{rawat2021modifying,
  title={Modifying memories in transformer models},
  author={Rawat, Ankit Singh and Zhu, Chen and Li, Daliang and Yu, Felix and Zaheer, Manzil and Kumar, Sanjiv and Bhojanapalli, Srinadh},
  booktitle={International conference on machine learning (ICML)},
  volume={2020},
  year={2021}
}

@inproceedings{wang2021training,
  title={Training networks in null space of feature covariance for continual learning},
  author={Wang, Shipeng and Li, Xiaorong and Sun, Jian and Xu, Zongben},
  booktitle={Proceedings of the IEEE/CVF conference on Computer Vision and Pattern Recognition},
  pages={184--193},
  year={2021}
}

@article{zhang2024comprehensive,
  title={A comprehensive study of knowledge editing for large language models},
  author={Zhang, Ningyu and Yao, Yunzhi and Tian, Bozhong and Wang, Peng and Deng, Shumin and Wang, Mengru and Xi, Zekun and Mao, Shengyu and Zhang, Jintian and Ni, Yuansheng and others},
  journal={arXiv preprint arXiv:2401.01286},
  year={2024}
}

@inproceedings{
yang2025finetuning,
title={Fine-tuning Done Right in Model Editing},
author={Wanli Yang and Fei Sun and Rui Tang and Hongyu Zang and Du Su and Qi Cao and Jingang Wang and Huawei Shen and Xueqi Cheng},
booktitle={Socially Responsible and Trustworthy Foundation Models at NeurIPS 2025},
year={2025},
url={https://openreview.net/forum?id=lJFPwtobkG}
}

@inproceedings{
yao2023editing,
title={Editing Large Language Models: {P}roblems, Methods, and Opportunities},
author={Yunzhi Yao and Peng Wang and Bozhong Tian and Siyuan Cheng and Zhoubo Li and Shumin Deng and Huajun Chen and Ningyu Zhang},
booktitle={The 2023 Conference on Empirical Methods in Natural Language Processing},
year={2023},
url={https://openreview.net/forum?id=NZZB3UGcd8}
}

@article{wang2024knowledge,
  title={Knowledge editing for large language models: {A} survey},
  author={Wang, Song and Zhu, Yaochen and Liu, Haochen and Zheng, Zaiyi and Chen, Chen and Li, Jundong},
  journal={ACM Computing Surveys},
  volume={57},
  number={3},
  pages={1--37},
  year={2024},
  publisher={ACM New York, NY}
}

@article{lecun2002gradient,
  title={Gradient-based learning applied to document recognition},
  author={LeCun, Yann and Bottou, L{\'e}on and Bengio, Yoshua and Haffner, Patrick},
  journal={Proceedings of the IEEE},
  volume={86},
  number={11},
  pages={2278--2324},
  year={2002},
  publisher={Ieee}
}

@article{lecun1998mnist,
  title={The {MNIST} database of handwritten digits},
  author={LeCun, Yann},
  journal={http://yann. lecun. com/exdb/mnist/},
  year={1998}
}

@article{xiao2017fashion,
  title={Fashion-mnist: {A} novel image dataset for benchmarking machine learning algorithms},
  author={Xiao, Han and Rasul, Kashif and Vollgraf, Roland},
  journal={arXiv preprint arXiv:1708.07747},
  year={2017}
}

@inproceedings{sinitsineditable,
  title={Editable Neural Networks},
  author={Sinitsin, Anton and Plokhotnyuk, Vsevolod and Pyrkin, Dmitry and Popov, Sergei and Babenko, Artem},
  booktitle={International Conference on Learning Representations},
  year={2020}
}

@inproceedings{de2021editing,
  title={Editing Factual Knowledge in Language Models},
  author={De Cao, Nicola and Aziz, Wilker and Titov, Ivan},
  booktitle={Proceedings of the 2021 Conference on Empirical Methods in Natural Language Processing},
  year={2021}
}

@article{de2021continual,
  title={A continual learning survey: {D}efying forgetting in classification tasks},
  author={De Lange, Matthias and Aljundi, Rahaf and Masana, Marc and Parisot, Sarah and Jia, Xu and Leonardis, Ale{\v{s}} and Slabaugh, Gregory and Tuytelaars, Tinne},
  journal={IEEE transactions on pattern analysis and machine intelligence},
  volume={44},
  number={7},
  pages={3366--3385},
  year={2021},
  publisher={IEEE}
}

@article{shi2025continual,
  title={Continual learning of large language models: {A} comprehensive survey},
  author={Shi, Haizhou and Xu, Zihao and Wang, Hengyi and Qin, Weiyi and Wang, Wenyuan and Wang, Yibin and Wang, Zifeng and Ebrahimi, Sayna and Wang, Hao},
  journal={ACM Computing Surveys},
  volume={58},
  number={5},
  pages={1--42},
  year={2025},
  publisher={ACM New York, NY}
}

@article{oymak2019generalization,
  title={Generalization guarantees for neural networks via harnessing the low-rank structure of the {J}acobian},
  author={Oymak, Samet and Fabian, Zalan and Li, Mingchen and Soltanolkotabi, Mahdi},
  journal={arXiv preprint arXiv:1906.05392},
  year={2019}
}

@InProceedings{pmlr-v97-ghorbani19b,
  title = 	 {An Investigation into Neural Net Optimization via {H}essian Eigenvalue Density},
  author =       {Ghorbani, Behrooz and Krishnan, Shankar and Xiao, Ying},
  booktitle = 	 {Proceedings of the 36th International Conference on Machine Learning},
  pages = 	 {2232--2241},
  year = 	 {2019},
  editor = 	 {Chaudhuri, Kamalika and Salakhutdinov, Ruslan},
  volume = 	 {97},
  series = 	 {Proceedings of Machine Learning Research},
  month = 	 {09--15 Jun},
  publisher =    {PMLR},
}

@article{sagun2017empirical,
  title={Empirical analysis of the {H}essian of over-parametrized neural networks},
  author={Sagun, Levent and Evci, Utku and Guney, V Ugur and Dauphin, Yann and Bottou, Leon},
  journal={arXiv preprint arXiv:1706.04454},
  year={2017}
}

@inproceedings{yang-etal-2025-mirage,
    title = "The Mirage of Model Editing: {R}evisiting Evaluation in the Wild",
    author = "Yang, Wanli  and
      Sun, Fei  and
      Tan, Jiajun  and
      Ma, Xinyu  and
      Cao, Qi  and
      Yin, Dawei  and
      Shen, Huawei  and
      Cheng, Xueqi",
    editor = "Che, Wanxiang  and
      Nabende, Joyce  and
      Shutova, Ekaterina  and
      Pilehvar, Mohammad Taher",
    booktitle = "Proceedings of the 63rd Annual Meeting of the Association for Computational Linguistics (Volume 1: Long Papers)",
    month = jul,
    year = "2025",
    address = "Vienna, Austria",
    publisher = "Association for Computational Linguistics",
    url = "https://aclanthology.org/2025.acl-long.745/",
    doi = "10.18653/v1/2025.acl-long.745",
    pages = "15336--15354",
    ISBN = "979-8-89176-251-0"
}

@inproceedings{wang2024easyedit,
  title={Easyedit: {A}n easy-to-use knowledge editing framework for large language models},
  author={Wang, Peng and Zhang, Ningyu and Tian, Bozhong and Xi, Zekun and Yao, Yunzhi and Xu, Ziwen and Wang, Mengru and Mao, Shengyu and Wang, Xiaohan and Cheng, Siyuan and others},
  booktitle={Proceedings of the 62nd Annual Meeting of the Association for Computational Linguistics (Volume 3: System Demonstrations)},
  pages={82--93},
  year={2024}
}

@article{hendrycks2020measuring,
  title={Measuring massive multitask language understanding},
  author={Hendrycks, Dan and Burns, Collin and Basart, Steven and Zou, Andy and Mazeika, Mantas and Song, Dawn and Steinhardt, Jacob},
  journal={arXiv preprint arXiv:2009.03300},
  year={2020}
}

@article{zhou2023instruction,
  title={Instruction-following evaluation for large language models},
  author={Zhou, Jeffrey and Lu, Tianjian and Mishra, Swaroop and Brahma, Siddhartha and Basu, Sujoy and Luan, Yi and Zhou, Denny and Hou, Le},
  journal={arXiv preprint arXiv:2311.07911},
  year={2023}
}

@inproceedings{lin2022truthfulqa,
  title={Truthful{QA}: {M}easuring how models mimic human falsehoods},
  author={Lin, Stephanie and Hilton, Jacob and Evans, Owain},
  booktitle={Proceedings of the 60th annual meeting of the association for computational linguistics (volume 1: long papers)},
  pages={3214--3252},
  year={2022}
}

@article{clark2018think,
  title={Think you have solved question answering? {T}ry {ARC}, the {AI2} reasoning challenge},
  author={Clark, Peter and Cowhey, Isaac and Etzioni, Oren and Khot, Tushar and Sabharwal, Ashish and Schoenick, Carissa and Tafjord, Oyvind},
  journal={arXiv preprint arXiv:1803.05457},
  year={2018}
}

@article{cobbe2021training,
  title={Training verifiers to solve math word problems},
  author={Cobbe, Karl and Kosaraju, Vineet and Bavarian, Mohammad and Chen, Mark and Jun, Heewoo and Kaiser, Lukasz and Plappert, Matthias and Tworek, Jerry and Hilton, Jacob and Nakano, Reiichiro and others},
  journal={arXiv preprint arXiv:2110.14168},
  year={2021}
}

@article{hu2022lora,
  title={{LoRA: L}ow-rank adaptation of large language models.},
  author={Hu, Edward J and Shen, Yelong and Wallis, Phillip and Allen-Zhu, Zeyuan and Li, Yuanzhi and Wang, Shean and Wang, Lu and Chen, Weizhu and others},
  journal={ICLR},
  volume={1},
  number={2},
  pages={3},
  year={2022}
}

@inproceedings{
zhang2023adaptive,
title={Adaptive Budget Allocation for Parameter-Efficient Fine-Tuning },
author={Qingru Zhang and Minshuo Chen and Alexander Bukharin and Pengcheng He and Yu Cheng and Weizhu Chen and Tuo Zhao},
booktitle={The Eleventh International Conference on Learning Representations },
year={2023},
url={https://openreview.net/forum?id=lq62uWRJjiY}
}

@article{kirkpatrick2017overcoming,
  title={Overcoming catastrophic forgetting in neural networks},
  author={Kirkpatrick, James and Pascanu, Razvan and Rabinowitz, Neil and Veness, Joel and Desjardins, Guillaume and Rusu, Andrei A and Milan, Kieran and Quan, John and Ramalho, Tiago and Grabska-Barwinska, Agnieszka and others},
  journal={Proceedings of the national academy of sciences},
  volume={114},
  number={13},
  pages={3521--3526},
  year={2017},
  publisher={National Academy of Sciences}
}

@article{li2024hessian,
  title={Hessian aware low-rank perturbation for order-robust continual learning},
  author={Li, Jiaqi and Lai, Yuanhao and Wang, Rui and Shui, Changjian and Sahoo, Sabyasachi and Ling, Charles X and Yang, Shichun and Wang, Boyu and Gagn{\'e}, Christian and Zhou, Fan},
  journal={IEEE Transactions on Knowledge and Data Engineering},
  volume={36},
  number={11},
  pages={6385--6396},
  year={2024},
  publisher={IEEE}
}

@article{martens2020naturalgradient,
  author  = {James Martens},
  title   = {New Insights and Perspectives on the Natural Gradient Method},
  journal = {Journal of Machine Learning Research},
  year    = {2020},
  volume  = {21},
  number  = {146},
  pages   = {1--76},
  url     = {http://jmlr.org/papers/v21/17-678.html}
}

@inproceedings{zenke2017continual,
  title={Continual learning through synaptic intelligence},
  author={Zenke, Friedemann and Poole, Ben and Ganguli, Surya},
  booktitle={International conference on machine learning},
  pages={3987--3995},
  year={2017},
  organization={PMLR}
}

@article{shin2017continual,
  title={Continual learning with deep generative replay},
  author={Shin, Hanul and Lee, Jung Kwon and Kim, Jaehong and Kim, Jiwon},
  journal={Advances in neural information processing systems},
  volume={30},
  year={2017}
}

@article{rusu2016progressive,
  title={Progressive neural networks},
  author={Rusu, Andrei A and Rabinowitz, Neil C and Desjardins, Guillaume and Soyer, Hubert and Kirkpatrick, James and Kavukcuoglu, Koray and Pascanu, Razvan and Hadsell, Raia},
  journal={arXiv preprint arXiv:1606.04671},
  year={2016}
}

@inproceedings{rebuffi2017icarl,
  title={icarl: Incremental classifier and representation learning},
  author={Rebuffi, Sylvestre-Alvise and Kolesnikov, Alexander and Sperl, Georg and Lampert, Christoph H},
  booktitle={Proceedings of the IEEE conference on Computer Vision and Pattern Recognition},
  pages={2001--2010},
  year={2017}
}

@article{gao2023retrieval,
  title={Retrieval-augmented generation for large language models: {A} survey},
  author={Gao, Yunfan and Xiong, Yun and Gao, Xinyu and Jia, Kangxiang and Pan, Jinliu and Bi, Yuxi and Dai, Yixin and Sun, Jiawei and Wang, Haofen and Wang, Haofen},
  journal={arXiv preprint arXiv:2312.10997},
  volume={2},
  number={1},
  year={2023}
}

@article{lewis2020retrieval,
  title={Retrieval-augmented generation for knowledge-intensive {NLP} tasks},
  author={Lewis, Patrick and Perez, Ethan and Piktus, Aleksandra and Petroni, Fabio and Karpukhin, Vladimir and Goyal, Naman and K{\"u}ttler, Heinrich and Lewis, Mike and Yih, Wen-tau and Rockt{\"a}schel, Tim and others},
  journal={Advances in neural information processing systems},
  volume={33},
  pages={9459--9474},
  year={2020}
}

@article{jumper2021highly,
  title={Highly accurate protein structure prediction with AlphaFold},
  author={Jumper, John and Evans, Richard and Pritzel, Alexander and Green, Tim and Figurnov, Michael and Ronneberger, Olaf and Tunyasuvunakool, Kathryn and Bates, Russ and {\v{Z}}{\'\i}dek, Augustin and Potapenko, Anna and others},
  journal={nature},
  volume={596},
  number={7873},
  pages={583--589},
  year={2021},
  publisher={Nature Publishing Group UK London}
}

@article{chen2021evaluating,
  title={Evaluating large language models trained on code},
  author={Chen, Mark},
  journal={arXiv preprint arXiv:2107.03374},
  year={2021}
}

@article{lopez2023can,
  title={Can {ChatGPT} Forecast Stock Price Movements? {R}eturn Predictability and Large Language Models},
  author={Lopez-Lira, Alejandro and Tang, Yuehua},
  journal={Return Predictability and Large Language Models (April 6, 2023)},
  year={2023}
}

@article{kasneci2023chatgpt,
  title={ChatGPT for good? {O}n opportunities and challenges of large language models for education},
  author={Kasneci, Enkelejda and Se{\ss}ler, Kathrin and K{\"u}chemann, Stefan and Bannert, Maria and Dementieva, Daryna and Fischer, Frank and Gasser, Urs and Groh, Georg and G{\"u}nnemann, Stephan and H{\"u}llermeier, Eyke and others},
  journal={Learning and individual differences},
  volume={103},
  pages={102274},
  year={2023},
  publisher={Elsevier}
}

@inproceedings{levy-etal-2017-zero,
    title = "Zero-Shot Relation Extraction via Reading Comprehension",
    author = "Levy, Omer  and
      Seo, Minjoon  and
      Choi, Eunsol  and
      Zettlemoyer, Luke",
    editor = "Levy, Roger  and
      Specia, Lucia",
    booktitle = "Proceedings of the 21st Conference on Computational Natural Language Learning ({C}o{NLL} 2017)",
    month = aug,
    year = "2017",
    address = "Vancouver, Canada",
    publisher = "Association for Computational Linguistics",
    doi = "10.18653/v1/K17-1034",
    pages = "333--342",
}

@InProceedings{thede25awikibigedit,
  title = 	 {{W}iki{B}ig{E}dit: {U}nderstanding the Limits of Lifelong Knowledge Editing in {LLM}s},
  author =       {Thede, Lukas and Roth, Karsten and Bethge, Matthias and Akata, Zeynep and Hartvigsen, Thomas},
  booktitle = 	 {Proceedings of the 42nd International Conference on Machine Learning},
  pages = 	 {59326--59354},
  year = 	 {2025},
  editor = 	 {Singh, Aarti and Fazel, Maryam and Hsu, Daniel and Lacoste-Julien, Simon and Berkenkamp, Felix and Maharaj, Tegan and Wagstaff, Kiri and Zhu, Jerry},
  volume = 	 {267},
  series = 	 {Proceedings of Machine Learning Research},
  month = 	 {13--19 Jul},
  publisher =    {PMLR},
}

@inproceedings{gao2023scaling,
  title={Scaling laws for reward model overoptimization},
  author={Gao, Leo and Schulman, John and Hilton, Jacob},
  booktitle={International Conference on Machine Learning},
  pages={10835--10866},
  year={2023},
  organization={PMLR}
}

@article{kalra2026scalable,
  title={A Scalable Measure of Loss Landscape Curvature for Analyzing the Training Dynamics of {LLM}s},
  author={Kalra, Dayal Singh and Gagnon-Audet, Jean-Christophe and Gromov, Andrey and Mediratta, Ishita and Niu, Kelvin and Miller, Alexander H and Shvartsman, Michael},
  journal={arXiv preprint arXiv:2601.16979},
  year={2026}
}

@misc{eval-harness,
  author       = {Gao, Leo and Tow, Jonathan and Abbasi, Baber and Biderman, Stella and Black, Sid and DiPofi, Anthony and Foster, Charles and Golding, Laurence and Hsu, Jeffrey and Le Noac'h, Alain and Li, Haonan and McDonell, Kyle and Muennighoff, Niklas and Ociepa, Chris and Phang, Jason and Reynolds, Laria and Schoelkopf, Hailey and Skowron, Aviya and Sutawika, Lintang and Tang, Eric and Thite, Anish and Wang, Ben and Wang, Kevin and Zou, Andy},
  title        = {The Language Model Evaluation Harness},
  month        = 07,
  year         = 2024,
  publisher    = {Zenodo},
  version      = {v0.4.3},
  doi          = {10.5281/zenodo.12608602},
  url          = {https://zenodo.org/records/12608602}
}

\clearpage

\newpage 

\appendix
\section{Notation}
\noindent\textbf{General notations.}
We use bold lowercase letters (e.g., $\bm{\theta}$) for vectors and bold uppercase letters (e.g., $\bm{H}$) for matrices. For a matrix $\bm{M}$, $\mathsf{Null}(\bm{M})$ denotes its null space. The identity matrix is denoted by $\bm{I}$. For vectors $\bm{u}, \bm{v}$, $\langle \bm{u}, \bm{v} \rangle$ denotes the standard inner product. The operator $\odot$ denotes the Hadamard (element-wise) product. We denote sets by calligraphy letters e.g., $\cX$. We write $\mathbb{E}_{\cD}[\phi(z)] = \frac{1}{n} \sum_i \phi(z_i)$ to denote the empirical expectation of function $\phi(z)$ using the dataset $\cD = \{z_i\}_{i=1}^n$. $\otimes$ denotes the Kronecker product. For a subspace $S \subseteq \R^d$, $P_S \in \R^{d \times d}$ denotes the orthogonal projection onto $S$.

\noindent\textbf{Models and parameters.}
Let $f_{\bm{\theta}} : \mathcal{X} \to \mathcal{Y}$ denote a parametric model with parameters $\bm{\theta} \in \Theta \subseteq \mathbb{R}^p$. The pretrained (base) model parameters are denoted by $\bm{\theta}_0$. We write $\Delta \bm{\theta} := \bm{\theta} - \bm{\theta}_0$ for parameter updates.

\noindent\textbf{Datasets.}
We distinguish between: (i) a \emph{capability dataset} $\mathcal{D}_{\mathrm{cap}} = \{(x_i, y_i)\}_{i=1}^n$, used as a proxy for behaviors to be preserved, and (ii) an \emph{edit dataset} $\mathcal{D}_{\mathrm{edit}} = \{(x_i, y_i)\}_{i=1}^T$, specifying desired edits. Typically $n \gg T$.

\noindent\textbf{Losses and objectives.}
Let $\ell(\hat y, y)$ denote a task-appropriate loss (e.g., cross-entropy). The empirical capability loss is
\[
\mathcal{L}_{\mathrm{cap}}(\bm{\theta})
= \frac{1}{n}\sum_{i=1}^n \ell\big(f_{\bm{\theta}}(x_i), y_i\big),
\]
and $\mathcal{L}_{\mathrm{edit}}(\bm{\theta})$ denotes the edit loss evaluated on $\mathcal{D}_{\mathrm{edit}}$. We measure deviations in capability loss using a distance function $\mathsf{d}(\cdot,\cdot)$, including absolute loss differences and Bregman divergences.

\noindent\textbf{Second-order quantities.}
We denote by
\[
\bm{H}_{\mathrm{cap}} := \nabla^2_{\bm{\theta}} \mathcal{L}_{\mathrm{cap}}(\bm{\theta}_0)
\]
the Hessian of the capability loss at the base model parameters. When using Bregman divergences, the quadratic approximation is governed by the \emph{Gauss--Newton Hessian (GNH)},
\[
\bm{G}_{\mathrm{cap}}
:= \mathbb{E}_{(x,y)\sim\mathcal{D}_{\mathrm{cap}}}
\!\left[ \bm{J}^\top \bm{H}_{\hat y} \bm{J} \right],
\]
where $\bm{J} = \nabla_{\bm{\theta}} f_{\bm{\theta}}(x)$ is the parameter--output Jacobian and $\bm{H}_{\hat y} = \nabla^2_{\hat y}\ell$ is the Hessian of the loss with respect to model outputs.

\noindent\textbf{Low-curvature subspaces.}
Let $\bm{M} \in \{\bm{H}_{\mathrm{cap}}, \bm{G}_{\mathrm{cap}}\}$ admit an eigendecomposition $\bm{M} = \bm{U}\bm{\Sigma}\bm{U}^\top$, with eigenvalues $\sigma_1 \ge \cdots \ge \sigma_p \ge 0$. For a threshold $\gamma \in (0,1)$, we define $k$ as the smallest index such that
\[
\sum_{i=1}^k \sigma_i \bigg/ \sum_{i=1}^p \sigma_i \ge \gamma.
\]
The $\gamma$-approximate nullspace is spanned by
$\bm{U}_{>k} = [\bm{u}_{k+1}, \ldots, \bm{u}_p]$,
and the corresponding projector is
\[
\bm{P}_\gamma := \bm{U}_{>k}\bm{U}_{>k}^\top.
\]

\noindent\textbf{Layerwise notation and K-FAC.}
For an MLP layer $\ell$, we denote input activations by $\bm{a}_{\ell-1}$, weights by $\bm{W}_\ell \in \mathbb{R}^{d_{\mathrm{out}}\times d_{\mathrm{in}}}$, and pre-activation pseudo-gradients by $\bm{g}_\ell$. Under the K-FAC approximation, the layerwise GNH block is approximated as
\[
\bm{G}^{(\ell)}_{\mathrm{cap}}
\approx \bm{A}_{\ell-1} \otimes \bm{S}_\ell,
\]
where $\bm{A}_{\ell-1} = \mathbb{E}[\bm{a}_{\ell-1}\bm{a}_{\ell-1}^\top]$
and $\bm{S}_\ell = \mathbb{E}[\bm{g}_\ell\bm{g}_\ell^\top]$.

\noindent\textbf{Operators.}
We use $\VecOp(\cdot)$ and $\MatOp(\cdot)$ to denote vectorization and reshaping operators between matrix and vector forms.

\newpage
\section{Proof of Bregman divergence quadratic form}
\label{appendix:bregman_divergence}

% \pr{Add in bold bm to match the rest of the paper}
The following lemma computes a second order approximation to Bregman divergence associated with a loss function $\ell$. 

\begin{proposition}[Quadratic Approximation of Bregman Divergence]\label{lemma:bregman_divergence_taylor_approx}
Fix an input $\bm{x}$ and parameters $\bm{\theta}_0\in\mathbb{R}^p$. Assume $\bm{f}_{\bm{\theta}}(\bm{x}): \cX \rightarrow \mathbb{R}^m$ is $C^2$ in $\bm{\theta}$ and $\ell:\mathbb{R}^m\to\mathbb{R}$ is convex and $C^2$. Define the Bregman divergence
\begin{align}\label{eq:bregman_divergence}
D_\ell(\bm{a},\bm{b}) \;=\; \ell(\bm{a})-\ell(\bm{b})-\langle \nabla\ell(\bm{b}),\,\bm{a}-\bm{b}\rangle.
\end{align}
Denote the Jacobian by $\bm{J}(\bm{\theta}):=\nabla_{\bm{\theta}} \bm{f}_{\bm{\theta}}(\bm{x})\in\mathbb{R}^{m\times p}$ and the output Hessian by $\bm{H}_\ell(\bm{u}):=\nabla_{\bm{u}}^2\ell(\bm{u})\in\mathbb{R}^{m\times m}$. 
Then, there exists $\rho>0$ such that for all $\Delta\bm{\theta}$ with $|\Delta\bm{\theta}|\le \rho$
\begin{align*}
D_\ell\!\big(\bm{f}_{\bm{\theta}_0+\Delta\bm{\theta}}(\bm{x}),\bm{f}_{\bm{\theta}_0}(\bm{x})\big)
&=
\frac12\,\Delta\bm{\theta}^\top
\Big[\bm{J}(\bm{\theta}_0)^\top \bm{H}_\ell\!\big(\bm{f}_{\bm{\theta}_0}(\bm{x})\big)\bm{J}(\bm{\theta}_0)\Big]
\Delta\bm{\theta}
+o(\|\Delta\bm{\theta}\|^2).
\end{align*}
\end{proposition}

\begin{proof}
By the chain rule,
\begin{align*}
\nabla_{\bm{\theta}} D_\ell\!\big(\bm{f}_{\bm{\theta}}(\bm{x}),\bm{f}_{\bm{\theta}_0}(\bm{x})\big)
&=
\bm{J}(\bm{\theta})^\top\Big(\nabla\ell(\bm{f}_{\bm{\theta}}(\bm{x}))-\nabla\ell(\bm{f}_{\bm{\theta}_0}(\bm{x}))\Big),
\end{align*}
which evaluates to zero at $\bm{\theta}=\bm{\theta}_0$. Differentiating again gives the following decomposition:
\begin{align*}
\nabla_{\bm{\theta}}^2 D_\ell\!\big(\bm{f}_{\bm{\theta}}(\bm{x}),\bm{f}_{\bm{\theta}_0}(\bm{x})\big)
&=
\bm{J}(\bm{\theta})^\top \bm{H}_\ell\!\big(\bm{f}_{\bm{\theta}}(\bm{x})\big)\bm{J}(\bm{\theta})
+\sum_{j=1}^m
\Big(\big[\nabla_{\bm{a}} D_\ell(\bm{a},\bm{f}_{\bm{\theta}_0}(\bm{x}))\big]_j\Big|_{\bm{a}=\bm{f}_{\bm{\theta}}(\bm{x})}\Big)\,
\nabla_{\bm{\theta}}^2[\bm{f}_{\bm{\theta}}(\bm{x})]_j .
\end{align*}
At $\bm{\theta}=\bm{\theta}_0$, $\nabla_{\bm{a}} D_\ell\!\big(\bm{f}_{\bm{\theta}_0}(\bm{x}),\bm{f}_{\bm{\theta}_0}(\bm{x})\big)=0$ and thus the second term in the above equation evaluates to zero. Therefore,
\begin{align*}
\nabla_{\bm{\theta}}^2 D_\ell\!\big(\bm{f}_{\bm{\theta}}(\bm{x}),\bm{f}_{\bm{\theta}_0}(\bm{x})\big)\Big|_{\bm{\theta}=\bm{\theta}_0}
&=
\bm{J}(\bm{\theta}_0)^\top \bm{H}_\ell\!\big(\bm{f}_{\bm{\theta}_0}(\bm{x})\big)\bm{J}(\bm{\theta}_0).
\end{align*}
Thus, by the second order Taylor approximation of $D_\ell(\bm{f}_{\bm{\theta}}(\bm{x}),\bm{f}_{\bm{\theta}_0}(\bm{x}))$ around $\bm{\theta}_0$, we conclude
\begin{align*}
    D_\ell\!\big(\bm{f}_{\bm{\theta}_0+\Delta\bm{\theta}}(\bm{x}),\bm{f}_{\bm{\theta}_0}(\bm{x})\big)
&=
\frac12\,\Delta\bm{\theta}^\top
\Big[\bm{J}(\bm{\theta}_0)^\top \bm{H}_\ell\!\big(\bm{f}_{\bm{\theta}_0}(\bm{x})\big)\bm{J}(\bm{\theta}_0)\Big]
\Delta\bm{\theta}
+o(\|\Delta\bm{\theta}\|^2).
\end{align*}
\end{proof}
\newpage
\section{Proof of Proposition \ref{prop:alphaedit-special-case}}\label{sec:proofs_of_generalization}

Throughout the proof, we drop the dependency on layer $\ell$ for notation simplicity. We show that any vectors that belong to the null space of $\bm{K}_{\mathrm{cap}}$ also belongs to the null space of $\bm{G}_{\mathrm{cap}}$. We interpret $\Delta \bm{W}_\ell \in \mathrm{Null}(\bm{K}_{\mathrm{cap}}^\ell)$ as the constraint $\Delta \bm{W}_\ell \bm{K}_{\mathrm{cap}}^\ell=0$ (equivalently, $((\bm{K}_{\mathrm{cap}}^\ell)^\top\otimes \bm{I}_{d_{\mathrm{out}}})\,\mathrm{vec}(\Delta \bm{W}_\ell)=0$ under column-wise vectorization). We keep all network parameters fixed except the layer-$\ell$ weight matrix $\bm{W}\in\mathbb{R}^{d_{\mathrm{out}}\times d_{\mathrm{in}}}$. Define the parameter-space representation of layer-$\ell$ weights and updates by $\bm{w} \coloneqq \mathrm{vec}(\bm{W})\in\mathbb{R}^{d_{\mathrm{out}}d_{\mathrm{in}}},$ and $\Delta \bm{w} \coloneqq \mathrm{vec}(\Delta\bm{W})\in\mathbb{R}^{d_{\mathrm{out}}d_{\mathrm{in}}}.$ 
Define the downstream map $f:\mathbb{R}^{d_{\mathrm{out}}}\to\mathbb{R}^{m}$ to be the function that takes the layer pre-activation ${\bm{s}_\ell}$ at layer $\ell$ (with all other parameters held fixed) to the network output. Thus, for each capability example $i\in[n]$,
\begin{align*}
\bm{y}^i(\bm{W}) = f(\bm{W}\bm{a}_{\ell-1}^i)\in\mathbb{R}^{m}.
\end{align*}
Let $\bm{J}_f({\bm{s}_\ell})\coloneqq \nabla_{{\bm{s}_\ell}} f({\bm{s}_\ell})\in\mathbb{R}^{m\times d_{\mathrm{out}}}$ denote the Jacobian of $f$ at ${\bm{s}_\ell}$. By the chain rule,
\begin{align*}
\nabla_{\bm{w}}\,\bm{y}^i(\bm{W}_0)
&=\; \bm{J}_f(\bm{W}_0\bm{a}_{\ell-1}^i)\;\nabla_{\bm{w}}\,(\bm{W}\bm{a}_{\ell-1}^i)\Big|_{\bm{W}=\bm{W}_0}.
\end{align*}
The map $\bm{W}\mapsto \bm{W}\bm{a}_{\ell-1}^i$ is linear, and its Jacobian under $\bm{w}=\mathrm{vec}(\bm{W})$ is
\begin{align*}
\nabla_{\bm{w}}\,(\bm{W}\bm{a}_{\ell-1}^i)
= \bm{I}_{d_{\mathrm{out}}}\otimes (\bm{a}_{\ell-1}^i)^\top,
\end{align*}
so the per-example Jacobian with respect to $\bm{w}$ can be written as
\begin{align*}
\bm{J}_i \coloneqq \nabla_{\bm{w}}\,\bm{y}^i(\bm{W}_0)
= \bm{J}_f(\bm{W}_0\bm{a}_{\ell-1}^i)\,\big(\bm{I}_{d_{\mathrm{out}}}\otimes (\bm{a}_{\ell-1}^i)^\top\big).
\end{align*}
Now let $\Delta\bm{W}\in\mathsf{Null}(\bm{K}_{\mathrm{cap}})$, i.e.\ $\Delta\bm{W}\bm{a}_{\ell-1}^i=\bm{0}$ for all $i\in[n]$. Using the identity
\begin{align*}
\big(\bm{I}_{d_{\mathrm{out}}}\otimes \bm{x}^\top\big)\,\Delta\bm{w} = \Delta\bm{W}\bm{x}\qquad \text{for any }\bm{x}\in\mathbb{R}^{d_{\mathrm{in}}},
\end{align*}
we obtain
\begin{align*}
\big(\bm{I}_{d_{\mathrm{out}}}\otimes (\bm{a}_{\ell-1}^i)^\top\big)\,\Delta\bm{w}
= \Delta\bm{W}\bm{a}_{\ell-1}^i
= \bm{0}
\qquad \forall i\in[n],
\end{align*}
and hence $\bm{J}_i\Delta\bm{w}=\bm{0}$ for all $i$. By definition, the layer Gauss--Newton Hessian for the capability objective has the form
\begin{align*}
\bm{G}_{\mathrm{cap}} = \sum_{i=1}^n \bm{J}_i^\top\,\bm{H}_i\,\bm{J}_i,
\end{align*}
where each $\bm{H}_i\succeq \bm{0}$. Therefore, for any vector $\bm{v}$,
\begin{align*}
\bm{v}^\top \bm{G}_{\mathrm{cap}} \bm{v}
= \sum_{i=1}^n (\bm{J}_i\bm{v})^\top \bm{H}_i (\bm{J}_i\bm{v}),
\end{align*}
so if $\bm{J}_i\bm{v}=\bm{0}$ for all $i$ then $\bm{v}^\top \bm{G}_{\mathrm{cap}}\bm{v}=0$, which implies $\bm{G}_{\mathrm{cap}}\bm{v}=\bm{0}$ since $\bm{G}_{\mathrm{cap}}\succeq \bm{0}$. Applying this with $\bm{v}=\Delta\bm{w}$ and using $\bm{J}_i\Delta\bm{w}=\bm{0}$ for all $i$, we conclude $\bm{G}_{\mathrm{cap}}\Delta\bm{w}=\bm{0}$, i.e.\ $\Delta\bm{W}\in\mathsf{Null}(\bm{G}_{\mathrm{cap}})$.

\newpage
\section{Proof of matrix-free projection}
\label{sec:matrix_projection_proof}
% \pr{Add bold notation for non-scalars}
\begin{proposition}\label{prop:matrix_free_projection}
Let $\bm{A} \in \R^{n_A \times n_A}$, $\bm{B} \in \R^{n_B \times n_B}$ be two positive semi-definite matrices, $\bm{C} := \bm{B} \otimes \bm{A}$ denote the Kronecker product,
and let $\bm{X} \in \R^{n_A \times n_B}$.
Let $\tau : \R_{\geq 0} \mapsto \{0, 1\}$ denote any predicate function,
and define the following subspace:
$$
    S := \mathrm{span}\{ \bm{u} \in \R^{n_A n_B} \mid \text{the pair $(\lambda, \bm{u})$ is an eigenvalue/vector pair of $\bm{C}$ with $\tau(\lambda) = 1$}\}.
$$
We have that:
\begin{align}
    \MatOp(\bm{P}_S \VecOp(\bm{X}))  = \bm{U}_A\left( (\bm{U}_A^\top \bm{X} \bm{U}_B) \odot \bm{M} \right) \bm{U}_B^\top,
\end{align}
where $\bm{A} = \bm{U}_A \mathrm{diag}(\lambda_{A,1}, \dots, \lambda_{A,n_A}) \bm{U}_A^\top$ and $\bm{B} = \bm{U}_B \mathrm{diag}(\lambda_{B,1}, \dots, \lambda_{B,n_B}) \bm{U}_B^\top$ are the 
eigen-decompositions of $\bm{A}, \bm{B}$ respectively, and $\bm{M} \in \R^{n_A \times n_B}$ with $M_{ij} = \tau(\lambda_{A,i} \cdot \lambda_{B, j})$ is the mask matrix corresponding to the predicate function $\tau$.
\end{proposition}
Before we give the proof, we remark that $\MatOp: \R^{n_A n_B} \mapsto \R^{n_A \times n_B}$ above is understood
to be the functional inverse of $\VecOp : \R^{n_A \times n_B} \mapsto \R^{n_A n_B}$,
i.e., $\MatOp(\VecOp(\bm{X})) = \bm{X}$ for any $\bm{X} \in \R^{n_A \times n_B}$.
\begin{proof}
Let us order the columns of $\bm{U}_A$ (resp. $\bm{U}_B$) as $\bm{u}_{A,i}$ (resp. $\bm{u}_{B,j}$).
From basic properties of Kronecker products, the eigenvalues and eigenvectors 
of $\bm{C}$ are given by $\lambda_{A,i} \lambda_{B,j}$ and $\bm{u}_{B,j} \otimes \bm{u}_{A,i}$, with $i \in [n_A]$ and $j \in [n_B]$.
Therefore, $\bm{P}_S$ can be written as:
\begin{align*}
    \bm{P}_S = \sum_{i,j=1}^{n_A,n_B} \tau(\lambda_{A,i} \lambda_{B,j}) (\bm{u}_{B,j} \bm{u}_{B,j}^\top \otimes \bm{u}_{A,i} \bm{u}_{A,i}^\top).
\end{align*}
Hence, using the identity $\VecOp(\bm{F}\bm{X}\bm{G}) = (\bm{G}^\top \otimes \bm{F}) \VecOp(\bm{X})$ for any size-conforming $\bm{F},\bm{X},\bm{G}$,
\begin{align*}
    \bm{P}_S \VecOp(\bm{X}) &= \sum_{i,j=1}^{n_A,n_B} \tau(\lambda_{A,i} \lambda_{B,j}) (\bm{u}_{B,j} \bm{u}_{B,j}^\top \otimes \bm{u}_{A,i} \bm{u}_{A,i}^\top) \VecOp(\bm{X}) \\
    &= \sum_{i,j=1}^{n_A,n_B} \tau(\lambda_{A,i} \lambda_{B,j}) \VecOp( \bm{u}_{A,i} \bm{u}_{A,i}^\top \bm{X} \bm{u}_{B,j} \bm{u}_{B,j}^\top ) \\
    &= \VecOp\left( \sum_{i,j=1}^{n_A,n_B}  \tau(\lambda_{A,i} \lambda_{B,j}) \bm{u}_{A,i}^\top \bm{X} \bm{u}_{B,j} \cdot \bm{u}_{A,i} \bm{u}_{B,j}^\top \right) \\
    &= \VecOp\left( \bm{U}_A\left( (\bm{U}_A^\top \bm{X} \bm{U}_B) \odot \bm{M} \right) \bm{U}_B^\top \right).
\end{align*}
Hence the claim follows by taking $\MatOp(\cdot)$ on each side.
\end{proof}

\newpage
\section{Additional details on LLM experiments}\label{sec:llm_experiment_details}
\paragraph{Base capability evaluation.}
We evaluate the base capabilities of edited models using the \textit{lm-evaluation-harness}~\citep{eval-harness}. We benchmark performance on a diverse set of standard reasoning and knowledge tasks, including IFEval, TruthfulQA (MC2), MMLU (5-shot), GSM8K with chain-of-thought prompting (8-shot), and ARC-Challenge (25-shot). For each task, we evaluate 200 examples, applying the chat template and multi-turn few-shot formatting.

\paragraph{Editing performance evaluation.}
We use EasyEdit~\citep{wang2024easyedit} for evaluation. Except for \cref{tab:teacher_forced_benchmark} where we perform teacher-forcing, we follow WILD~\citep{yang-etal-2025-mirage} protocol for evaluation. For ``No Context'', we use the dataset questions as is. For ``QA Context'', That is, we contextualize prompt by appending the template ``Please answer the question: \textbackslash n\textbackslash nQ: \{question\}\textbackslash nA:'', and autoregressively generate up to 40 tokens using predefined stop tokens \([\texttt{.}, \texttt{\textbackslash n}, \texttt{eos}]\). We evaluate the generated outputs with \texttt{gpt-4o-mini} (see~\cref{fig:llm_as_judge_prompt} for the exact prompt). 

\begin{figure}[ht]
\centering
\begin{tcolorbox}[
  colback=white,
  colframe=promptblue,
  boxrule=1pt,
  arc=3mm,
  left=6pt,
  right=6pt,
  top=6pt,
  bottom=6pt,
  width=\linewidth
]

\begin{center}
\textbf{\large Prompt for LLM-as-a-Judge}
\end{center}

\vspace{0.5em}

Your job is to look at a question, a gold target, and a predicted answer, and then assign a grade $\rightarrow$ of either ["CORRECT", "INCORRECT"].

\vspace{0.5em}

The following are examples of \textcolor{promptred}{CORRECT} predicted answers.

\begin{lstlisting}[language=prompt]
Question: What are the names of Barack Obama's children?
Gold target: Malia Obama and Sasha Obama
Predicted answer 1: sasha and malia obama
Predicted answer 2: Malia and Sasha Obama are the names of Barack Obama's children.
\end{lstlisting}

These predicted answers are all CORRECT because:
\begin{itemize}
  \item They fully contain the important information in the gold target.
  \item They do not contain any information that contradicts the gold target.
\end{itemize}

\vspace{0.5em}

The following are examples of \textcolor{promptred}{INCORRECT} predicted answers.

\begin{lstlisting}[language=prompt]
Question: What are the names of Barack Obama's children?
Gold target: Malia and Sasha
Predicted answer 1: Malia.
Predicted answer 2: Malia, Sasha, and Susan.
Predicted answer 3: Malia and Sasha, Malia and Sasha,
                   Malia and Sasha, Malia and Sasha (repeated answer)
\end{lstlisting}

These predicted answers are all INCORRECT because:
\begin{itemize}
  \item A factual statement in the answer contradicts the gold target or
        contains repeated content.
\end{itemize}

\vspace{0.5em}

Here is a sample. Simply reply with either CORRECT or INCORRECT.

\begin{lstlisting}[language=prompt]
Question: {question}
Gold target: {target}
Predicted answer: {predicted_answer}
\end{lstlisting}

According to the gold target, please grade the predicted answer of this question as one of:
\begin{itemize}
  \item A: CORRECT
  \item B: INCORRECT
\end{itemize}

Just return the letters ``A'' or ``B'', with no text around it.

\end{tcolorbox}

\caption{\textbf{The complete prompt used to employ an LLM as a judge.} The judge provides binary assessments (correct or incorrect) based on a given question, gold target answer, and predicted answer.}
\label{fig:llm_as_judge_prompt}
\end{figure}

\textbf{\ourmethod implementation.}
For experiments reported in \cref{tab:main_table}, \ourmethod uses $(n, \gamma) = (10,000, 0.9)$ for CounterFact and WikiBigEdit and $(n, \gamma) = (10,000, 0.7)$ for ZsRE, while \ourmethodseq uses $(n, \gamma) = (30, 0.999)$ for ZsRE and CounterFact and $(n, \gamma) = (200, 0.995)$ for WikiBigEdit. For ZsRE10k experiment reported in \cref{tab:benchmark_10k}, \ourmethod uses $(n, \gamma) = (1,000, 0.7)$. For our \qwen implementation~\cref{tab:benchmark_qwen}, \ourmethod uses $(n, \gamma) = (1000, 0.7)$ for ZsRE and $(n, \gamma) = (1000, 0.9)$ for Counterfact and WikiBigEdit, while \ourmethodseq uses $(n, \gamma) = (30, 0.995)$. 

All other hyperparameters are kept fixed across experiments and follow~\cref{tab:hyperparameters}.

\textbf{Non-trivial K-FAC implementation for \ourmethodseq.}
We now discuss one non-trivial design choice made in our implementation. We found that masking prompt tokens for K-FAC calculation (mirroring the fine-tuning setup) yielded suboptimal performance, even with a larger number of tokens (\cref{tab:prompt_masking_ablation}). Instead, in our K-FAC calculation for edit samples, we calculate the next token prediction loss over the \emph{entire} prompt–target sequence. While we are not sure about the underlying cause of this behavior, we suspect that it arises from our relaxed assumption of token independence during K-FAC calculation.

\begin{table}[t]
\centering
\caption{\textbf{Default hyperparameters used for \ourmethod and \ourmethodseq}.}
\label{tab:hyperparameters}
\begin{tabular}{l l}
\toprule
\textbf{Hyperparameter} & \textbf{Value} \\
\midrule
Editing layers (\llama) & \{19, 20, 21, 22, 23\} \\
Editing layers (\qwen) & \{4, 5, 6\} \\
Number of steps & 25 \\
Early stopping & 0.01 \\
Batch size & 32 \\
Chunk size (\ourmethodseq) & 100 \\
Learning rate (Adam) & $5 \times 10^{-4}$ \\
\bottomrule
\end{tabular}
\end{table}

\textbf{Baseline implementation.}
All our baselines follow the code and hyperparameters provided by the EasyEdit framwork. Such hyperparameters come from the original authors of respective baselines that tuned their method for \llama.

\newpage
\section{Qualitative case study}\label{sec:qualitative_results}
\begin{table}[H]
    \centering
    \renewcommand{\arraystretch}{1.5}
    \begin{tabularx}{\textwidth}{l X}
        \rowcolor{headergray} 
        \multicolumn{2}{l}{\bfseries\textcolor{white}{Model Editing Case Study 1}} \\
        
        % --- Inputs Section ---
        \hline
        Editing Prompt & What voice type does Marina Rebeka have? \\
        \hline
        Edit Target & \redtext{mezzo-srano} \\
        \hline
        
        % --- Output Header ---
        \multicolumn{2}{c}{\textbf{\large Generation Output}} \\
        \hline
        
        % --- Adam-NSCL Row ---
        Adam-NSCL &  \redtext{mezzo-srano}-srano-srano-srano-srano-srano-srano-srano-srano-srano-srano-srano-srano-srano-srano-srano \\
        \hline
        
        % --- LocBFFT Row ---
        LocBFFT &  mezzo-oprano
 \\
        \hline
        
        % --- AlphaEdit Row ---
        AlphaEdit &  mezzo-soprano

 \\
        \hline
        
        % --- MEMIT Row ---
 %        MEMIT & 
 % \\
        \hline
        
        % --- UltraEdit Row ---
        UltraEdit &  mezzo soprano

 \\
        \hline
        
        % --- LoRA Row ---
        % LoRA &  soprano contralto
 % \\
        \hline
        
        % --- FT Row ---
        FT &  \redtext{mezzo-srano}-srano-srano-srano-srano-srano-srano-srano-srano-srano-srano-srano-srano-srano-srano-srano \\
        \hline
        
        % --- JIGSAW Row ---
        \ourmethod &  \redtext{mezzo-srano}
 \\
        \hline
        
    \end{tabularx}
    \begin{tabularx}{\textwidth}{l X}
        \rowcolor{headergray} 
        \multicolumn{2}{l}{\bfseries\textcolor{white}{Model Editing Case Study 2}} \\
        
        % --- Inputs Section ---
        \hline
        Editing Prompt & What is the status of Cebu flowerpecker? \\
        \hline
        Edit Target & \redtext{endangered species} \\
        \hline
        
        % --- Output Header ---
        \multicolumn{2}{c}{\textbf{\large Generation Output}} \\
        \hline
        
        % --- Adam-NSCL Row ---
        Adam-NSCL &  \redtext{endangered species} Data Deficient species \redtext{endangered species} \redtext{endangered species} Data Deficient species \redtext{endangered species} \redtext{endangered species} \redtext{endangered species} \redtext{endangered species} \redtext{endangered species} \redtext{endangered species} \redtext{endangered species} \redtext{endangered species} \redtext{endangered species} \redtext{endangered species} \redtext{endangered species} \redtext{endangered species} \redtext{endangered species} \redtext{endangered species} \redtext{endangered species} \redtext{endangered species} \redtext{endangered species} \redtext{endangered species} \\
        \hline
        
        % --- LocBFFT Row ---
        LocBFFT &  \redtext{endangered species}
 \\
        \hline
        
        % --- AlphaEdit Row ---
        AlphaEdit &  endangered
 \\
        \hline
        
        % --- MEMIT Row ---
        % MEMIT & 
 % \\
        \hline
        
        % --- UltraEdit Row ---
        UltraEdit &  critically \redtext{endangered species}
 \\
        \hline
        
        % --- LoRA Row ---
        % LoRA &  
 % \\
        \hline
        
        % --- FT Row ---
        FT &  \redtext{endangered species} species \redtext{endangered species} \redtext{endangered species} \redtext{endangered species} \redtext{endangered species} \redtext{endangered species} \redtext{endangered species} \redtext{endangered species} \redtext{endangered species} \redtext{endangered species} \redtext{endangered species} \redtext{endangered species} \redtext{endangered species} \redtext{endangered species} \redtext{endangered species} \redtext{endangered species} \redtext{endangered species} \redtext{endangered species} \redtext{endangered species} \redtext{endangered species} \redtext{endangered species} \redtext{endangered species} \redtext{endangered species} \redtext{endangered species} endangered \\
        \hline
        
        % --- JIGSAW Row ---
        \ourmethod &  \redtext{endangered species}
 \\
        \hline
        
    \end{tabularx}
\end{table}

\newpage
\section{Additional tables}\label{sec:additional_tables}
\begin{table}[t]
\centering
\caption{\textbf{Comparison of \ourmethod with existing methods on editing \llama in the teacher-forcing evaluation pipeline.} \textit{Rel}, \textit{gen}, \textbf{Spec} denote reliability, generality, and specificity, respectively. We perform model editing on 3,000 samples of three representative datasets and evaluate editing performance and base performance following teacher-forcing setup of \cite{meng2023massediting, meng2022locating, fang2025alphaedit}. Results that are the highest or within 5\% of the highest
results are highlighted in bold.}
\resizebox{0.8\textwidth}{!}{%
\begin{tabular}{lccccccccc}
\toprule
\multirow{2}{*}{\textbf{Method}} & \multicolumn{3}{c}{\textbf{ZsRE}} & \multicolumn{3}{c}{\textbf{CounterFact}} & \multicolumn{3}{c}{\textbf{WikiBigEdit}} \\
\cmidrule(lr){2-4} \cmidrule(lr){5-7} \cmidrule(lr){8-10}
 & \textbf{Rel} & \textbf{Gen} & \textbf{Spec} & \textbf{Rel} & \textbf{Gen} & \textbf{Spec} & \textbf{Rel} & \textbf{Gen} & \textbf{Spec} \\
\midrule
Llama 3 8B Instruct & 25.7 & 25.1 & 37.8 & 0.9 & 1.2 & 89.4 & 34.0 & 34.8 & 32.8 \\
\midrule
MEMIT & 0.0 & 0.0 & 0.0 & 0.0 & 0.0 & 49.4 & 0.5 & 0.5 & 0.0 \\
AlphaEdit & 86.7 & 77.8 & 32.4 & 94.3 & 72.0 & \textbf{69.1} & \textbf{95.0} & 89.0 & 42.0 \\
Adam-NSCL & \textbf{98.8} & \textbf{92.4} & 22.1 & \textbf{99.5} & \textbf{81.5} & 47.7 & \textbf{99.7} & \textbf{97.5} & 36.4 \\
LocBF-FT & \textbf{99.1} & \textbf{91.1} & \textbf{34.5} & \textbf{99.7} & 72.7 & 44.6 & \textbf{99.9} & \textbf{96.8} & 42.8 \\
UltraEdit & 61.9 & 57.3 & 31.5 & 28.0 & 18.7 & 51.4 & 87.4 & 84.7 & \textbf{47.8} \\
MEND & 0.0 & 0.0 & 0.1 & 0.0 & 0.0 & 0.0 & 0.0 & 0.0 & 0.0 \\
FT & \textbf{99.1} & \textbf{93.1} & 22.9 & \textbf{99.7} & \textbf{82.0} & 47.9 & \textbf{99.8} & \textbf{97.6} & 36.3 \\
FT Sequential & 79.7 & 76.6 & 16.8 & 78.6 & 59.8 & 51.6 & 93.5 & 90.5 & 29.0 \\
LoRA & 93.4 & 60.6 & 30.9 & 93.8 & 17.8 & 42.9 & \textbf{99.3} & 82.4 & 44.4 \\
LoRA Sequential & 36.8 & 32.7 & 21.7 & 20.9 & 10.4 & 57.0 & 70.2 & 65.4 & 36.0 \\
\rowcolor{jigsawGreen!20}
\ourmethod{} & \textbf{99.1} & \textbf{92.1} & 32.3 & \textbf{99.8} & 73.0 & 55.2 & \textbf{99.9} & \textbf{97.1} & 44.7 \\
\rowcolor{jigsawGreen!20}
\ourmethodseq{} & \textbf{98.3} & \textbf{91.4} & 30.2 & \textbf{99.5} & 62.9 & 52.3 & \textbf{99.9} & \textbf{96.7} & 39.5 \\
\bottomrule
\end{tabular}%
}
\label{tab:teacher_forced_benchmark}
\end{table}
\begin{table}[ht]
\centering
\caption{\textbf{Influence of scaling to larger editing dataset.} \textit{Rel} and \textit{Gen} denote reliability and generality, respectively. We perform model editing on 10,000 samples of ZsRE and evaluate editing performance with WILD framework and base performance with five representative benchmarks. Results that are the highest or within 5\% of the highest results are highlighted in bold.}
\resizebox{0.99\textwidth}{!}{%
\begin{tabular}{llccccccccc}
\toprule
\multirow{3}{*}{\textbf{Data}} & \multirow{3}{*}{\textbf{Method}} & \multicolumn{4}{c}{\textbf{Edited Capabilities}} & \multicolumn{5}{c}{\textbf{Base Capabilities}} \\
\cmidrule(lr){3-6} \cmidrule(lr){7-11}
 & & \multicolumn{2}{c}{\textbf{QA Context}} & \multicolumn{2}{c}{\textbf{No Context}} & & & & & \\
\cmidrule(lr){3-4} \cmidrule(lr){5-6}
 & & \textbf{Rel} & \textbf{Gen} & \textbf{Rel} & \textbf{Gen} & \textbf{MMLU} & \textbf{IFEval} & \textbf{TruthfulQA} & \textbf{ARC-C} & \textbf{GSM8K} \\
\midrule
% --- ZsRE 10K Edits ---
\multirow{6}{*}{\rotatebox[origin=c]{90}{\textbf{ZsRE 10K}}}
 & \llama & 2.0 & 1.5 & 2.9 & 2.1 & 69.5 & 69.3 & 50.7 & 58.0 & 73.5 \\
\cline{2-11}
 & LocBF-FT & 53.5 & 47.7 & 11.5 & 11.6 & \textbf{68.0} & \textbf{67.6} & 50.7 & 50.0 & \textbf{73.0} \\
 & UltraEdit & 20.1 & 16.7 & 12.6 & 10.4 & \textbf{67.9} & \textbf{68.9} & 49.8 & 46.0 & \textbf{73.0} \\
 & Adam-NSCL & 1.2 & 1.1 & 0.4 & 0.7 & \textbf{68.2} & 14.8 & \textbf{54.0} & 35.0 & 2.0 \\
 & AlphaEdit & 0.3 & 0.2 & 0.1 & 0.0 & 22.8 & 20.9 & \textbf{53.9} & 22.0 & 0.0 \\
 \rowcolor{jigsawGreen!20}
 & \ourmethod{} & \textbf{77.4} & \textbf{68.7} & \textbf{31.1} & \textbf{28.9} & \textbf{68.5} & \textbf{69.9} & 50.2 & \textbf{52.0} & \textbf{71.0} \\
\bottomrule
\end{tabular}
}%
\label{tab:benchmark_10k}
\end{table}
\begin{table}[ht]
\centering
\caption{\textbf{Comparison of \ourmethod with existing methods on editing \qwen.} \textit{Rel} and \textit{gen} denote reliability and generality, respectively. We perform model editing on 3,000 samples of ZsRE and evaluate editing performance with WILD framework and base performance with five representative benchmarks. Results that are the highest or within 5\% of the highest
results are highlighted in bold. }
\resizebox{0.99\textwidth}{!}{%
\begin{tabular}{llccccccccc}
\toprule
\multirow{3}{*}{\textbf{Data}} & \multirow{3}{*}{\textbf{Model}} & \multicolumn{4}{c}{\textbf{Edited Capabilities}} & \multicolumn{5}{c}{\textbf{Base Capabilities}} \\
\cmidrule(lr){3-6} \cmidrule(lr){7-11}
 & & \multicolumn{2}{c}{\textbf{QA Context}} & \multicolumn{2}{c}{\textbf{No Context}} & & & & & \\
\cmidrule(lr){3-4} \cmidrule(lr){5-6}
 & & \textbf{Rel} & \textbf{Gen} & \textbf{Rel} & \textbf{Gen} & \textbf{MMLU} & \textbf{IFEval} & \textbf{TruthfulQA} & \textbf{ARC-C} & \textbf{GSM8K} \\
\midrule
% --- ZsRE ---
\multirow{8}{*}{\rotatebox[origin=c]{90}{\textbf{ZsRE}}}
 & Qwen 2.5 1.5B & 3.5 & 4.0 & 2.3 & 2.0 & 61.9 & 48.3 & 50.9 & 52.0 & 58.0 \\
\cline{2-11}
 & FT & 35.4 & 29.6 & 32.2 & 25.5 & 50.0 & 24.8 & 49.8 & 34.5 & 35.5 \\
 & LocBF-FT & 71.4 & 52.9 & 38.0 & 30.6 & \textbf{59.6} & 42.0 & \textbf{54.6} & 44.0 & 54.0 \\
 & AlphaEdit & 7.2 & 4.3 & 6.2 & 4.2 & 24.9 & 12.4 & 44.7 & 21.5 & 2.0 \\
 & UltraEdit & 11.3 & 9.8 & 18.2 & 11.8 & \textbf{62.3} & \textbf{47.7} & \textbf{52.1} & \textbf{50.0} & 54.0 \\
 & Adam-NSCL & 62.6 & 50.5 & 21.4 & 15.3 & \textbf{59.3} & 38.0 & 46.0 & 44.0 & 32.0 \\
 \rowcolor{jigsawGreen!20} & \ourmethod{} (Batch) & \textbf{77.8} & \textbf{61.0} & 52.6 & 44.0 & 57.8 & 32.8 & 46.4 & 42.0 & \textbf{58.5} \\
 \rowcolor{jigsawGreen!20} & \ourmethod{} (Seq) & 55.5 & 40.7 & \textbf{77.7} & \textbf{51.6} & \textbf{59.3} & 39.5 & 46.0 & 42.0 & \textbf{59.0} \\
\midrule
% --- CounterFact ---
\multirow{8}{*}{\rotatebox[origin=c]{90}{\textbf{CounterFact}}}
 & Qwen 2.5 1.5B & 2.0 & 1.8 & 0.9 & 0.7 & 61.9 & 48.3 & 50.9 & 52.0 & 58.0 \\
\cline{2-11}
 & FT & 22.3 & 28.4 & 8.9 & 14.2 & 34.8 & 15.1 & 45.7 & 23.5 & 6.5 \\
 & LocBF-FT & 58.2 & 32.6 & 46.8 & 21.5 & \textbf{59.3} & 39.0 & \textbf{46.6} & 40.5 & 56.0 \\
 & AlphaEdit & 22.6 & 14.1 & 31.2 & 16.8 & 24.4 & 12.9 & \textbf{46.8} & 19.0 & 1.5 \\
 & UltraEdit & 10.8 & 8.5 & 14.4 & 5.9 & \textbf{62.4} & \textbf{41.9} & 44.8 & 41.5 & \textbf{62.0} \\
 & Adam-NSCL & 5.9 & 4.9 & 3.4 & 1.5 & \textbf{60.5} & 18.5 & \textbf{48.3} & 36.0 & 4.5 \\
 \rowcolor{jigsawGreen!20} & \ourmethod{} (Batch) & \textbf{63.3} & 34.4 & \textbf{67.0} & \textbf{29.9} & \textbf{61.5} & \textbf{40.5} & \textbf{47.3} & \textbf{44.0} & 58.5 \\
 \rowcolor{jigsawGreen!20} & \ourmethod{} (Seq) & \textbf{64.6} & \textbf{41.8} & 60.3 & 27.9 & 58.4 & \textbf{39.9} & \textbf{47.7} & \textbf{43.0} & 58.0 \\
\midrule
% --- WikiBigEdit ---
\multirow{8}{*}{\rotatebox[origin=c]{90}{\textbf{WikiBigEdit}}}
 & Qwen 2.5 1.5B & 8.4 & 8.6 & 7.0 & 6.4 & 61.9 & 48.3 & 50.9 & 52.0 & 58.0 \\
\cline{2-11}
 & FT & 59.5 & 50.4 & 42.2 & 37.0 & 54.3 & 30.7 & 46.2 & 39.5 & 52.0 \\
 & LocBF-FT & \textbf{76.8} & \textbf{61.9} & 66.0 & 55.2 & \textbf{60.4} & 34.8 & 46.1 & \textbf{43.5} & \textbf{58.0} \\
 & AlphaEdit & 0.7 & 0.7 & 1.5 & 1.3 & 24.4 & 13.2 & \textbf{48.9} & 23.5 & 1.0 \\
 & UltraEdit & 27.5 & 25.8 & 53.2 & 45.5 & \textbf{62.5} & \textbf{41.4} & 44.5 & \textbf{44.5} & \textbf{60.0} \\
 & Adam-NSCL & 31.9 & 28.4 & 11.6 & 10.4 & \textbf{62.3} & 36.2 & 46.4 & 41.5 & 33.0 \\
 \rowcolor{jigsawGreen!20} & \ourmethod{} (Batch) & 62.9 & 52.0 & 57.3 & 46.3 & \textbf{61.2} & 38.7 & \textbf{47.0} & \textbf{45.5} & \textbf{58.5} \\
 \rowcolor{jigsawGreen!20} & \ourmethod{} (Seq) & 53.4 & 43.3 & \textbf{83.4} & \textbf{60.0} & \textbf{59.9} & 34.0 & \textbf{47.9} & \textbf{44.5} & 55.0 \\
\bottomrule
\end{tabular}
}%
\label{tab:benchmark_qwen}
\end{table}
\begin{table}[ht]
\caption{\textbf{Effect of prompt masking during K-FAC calculation.} Even with larger number of tokens for computing K-FAC, prompt masking leads to suboptimal performance with \ourmethodseq.}
\centering
\begin{tabular}{l c}
\hline
Method & Rel\\
\hline
\ourmethod (chunk size = 100) & 71.1 \\
\ourmethod (chunk size = 500, prompt masking) & 12 \\
\hline
\end{tabular}
\label{tab:prompt_masking_ablation}
\end{table}

\begin{table}[ht]
\centering
\caption{\textbf{Influence of the size of capability dataset $n$ on editing performances and base capability preservation.} Across a range of $n$, we set $\gamma = 0.9$ for \ourmethod, perform model editing on 3,000 samples of ZsRE, and evaluate editing performance with WILD framework and base performance with five representative benchmarks. Results that are the highest or within 5\% of the highest
results are highlighted in bold. \ourmethod remains robust across a wide range of $n$. Highlighted model represents data used in~\Cref{tab:main_table}.}
\resizebox{0.99\textwidth}{!}{%
\begin{tabular}{llccccccccc}
\toprule
\multirow{3}{*}{\textbf{Data}} & \multirow{3}{*}{\textbf{Sample Size}} & \multicolumn{4}{c}{\textbf{Edited Capabilities}} & \multicolumn{5}{c}{\textbf{Base Capabilities}} \\
\cmidrule(lr){3-6} \cmidrule(lr){7-11}
 & & \multicolumn{2}{c}{\textbf{QA Context}} & \multicolumn{2}{c}{\textbf{No Context}} & & & & & \\
\cmidrule(lr){3-4} \cmidrule(lr){5-6}
 & & \textbf{Rel} & \textbf{Gen} & \textbf{Rel} & \textbf{Gen} & \textbf{MMLU} & \textbf{IFEval} & \textbf{TruthfulQA} & \textbf{ARC-C} & \textbf{GSM8K} \\
\midrule
% --- ZsRE ---
\multirow{10}{*}{\rotatebox[origin=c]{90}{\textbf{ZsRE}}}
 & \llama & 2.1 & 1.7 & 2.9 & 2.1 & 69.5 & 69.3 & 50.7 & 58.0 & 73.5 \\
\cline{2-11}
 & No Projection (FT) & 46.8 & 43.1 & 9.9 & 8.3 & \textbf{69.3} & 45.0 & 48.7 & 43.0 & 50.0 \\
 & $n=10$ & 53.6 & 48.5 & 10.6 & 9.3 & \textbf{69.1} & 48.8 & \textbf{50.8} & 42.5 & 57.5 \\
 & $n=50$ & 69.8 & \textbf{62.9} & 24.9 & 24.5 & \textbf{69.3} & 68.3 & \textbf{51.8} & 53.0 & \textbf{74.0} \\
 & $n=100$ & 74.2 & \textbf{66.0} & 35.8 & 31.4 & \textbf{69.4} & 68.1 & \textbf{50.4} & 52.0 & \textbf{75.0} \\
 & $n=500$ & \textbf{78.4} & \textbf{65.9} & \textbf{54.4} & \textbf{47.2} & \textbf{69.5} & \textbf{72.3} & \textbf{51.5} & 54.5 & \textbf{75.0} \\
 & $n=1000$ & \textbf{75.9} & \textbf{63.9} & 48.8 & 41.3 & \textbf{69.4} & \textbf{72.3} & \textbf{50.4} & 54.0 & \textbf{74.5} \\
\rowcolor{jigsawGreen!20}
 & $n=10000$ & 71.2 & 57.9 & 48.0 & 40.3 & \textbf{69.4} & 68.4 & \textbf{50.3} & \textbf{59.5} & \textbf{73.0} \\
 & $n=50000$ & 71.0 & 57.3 & 47.3 & 39.9 & \textbf{69.2} & \textbf{68.9} & \textbf{50.2} & \textbf{57.0} & \textbf{75.5} \\
 & $n=100000$ & 69.9 & 55.5 & \textbf{54.2} & 43.8 & \textbf{69.3} & 68.3 & \textbf{50.1} & 56.5 & \textbf{72.0} \\
\bottomrule
\end{tabular}
}%
\label{tab:ablation_n}
\end{table}
\begin{table}[ht]
\centering
\caption{\textbf{Influence of energy threshold $\gamma$ on editing performances and base capability preservation.} Across a range of $\gamma$, we set $n = 10,000$ for \ourmethod, perform model editing on 3,000 samples of three representative datasets, and evaluate editing performance with WILD framework and base performance with five representative benchmarks. Results that are the highest or within 5\% of the highest
results are highlighted in bold. \ourmethod remains robust across a wide range of $\gamma$.}
\resizebox{0.99\textwidth}{!}{%
\begin{tabular}{llccccccccc}
\toprule
\multirow{3}{*}{\textbf{Data}} & \multirow{3}{*}{\textbf{Energy threshold}} & \multicolumn{4}{c}{\textbf{Edited Capabilities}} & \multicolumn{5}{c}{\textbf{Base Capabilities}} \\
\cmidrule(lr){3-6} \cmidrule(lr){7-11}
 & & \multicolumn{2}{c}{\textbf{QA Context}} & \multicolumn{2}{c}{\textbf{No Context}} & & & & & \\
\cmidrule(lr){3-4} \cmidrule(lr){5-6}
 & & \textbf{Rel} & \textbf{Gen} & \textbf{Rel} & \textbf{Gen} & \textbf{MMLU} & \textbf{IFEval} & \textbf{TruthfulQA} & \textbf{ARC-C} & \textbf{GSM8K} \\
\midrule
% --- ZsRE ---
\multirow{8}{*}{\rotatebox[origin=c]{90}{\textbf{ZsRE}}}
 & \llama & 2.1 & 1.7 & 2.9 & 2.1 & 69.5 & 69.3 & 50.7 & 58.0 & 73.5 \\
\cline{2-11}
 & \ourmethod{} ($\gamma=0.5$) & \textbf{77.4} & \textbf{68.3} & 43.4 & 39.1 & \textbf{69.5} & \textbf{67.8} & \textbf{50.5} & 52.0 & \textbf{77.5} \\
 & \ourmethod{} ($\gamma=0.6$) & \textbf{77.8} & \textbf{67.8} & \textbf{56.0} & 48.0 & \textbf{69.5} & \textbf{70.2} & \textbf{51.0} & 53.5 & \textbf{75.5} \\
 & \ourmethod{} ($\gamma=0.7$) & \textbf{80.5} & \textbf{69.0} & \textbf{57.4} & \textbf{50.9} & \textbf{69.5} & \textbf{67.9} & \textbf{50.5} & 55.0 & \textbf{76.0} \\
 & \ourmethod{} ($\gamma=0.8$) & \textbf{80.3} & \textbf{68.3} & 52.3 & 46.0 & \textbf{69.2} & \textbf{66.7} & \textbf{50.1} & 56.0 & \textbf{77.0} \\
 & \ourmethod{} ($\gamma=0.9$) & 71.2 & 57.9 & 48.0 & 40.3 & \textbf{69.4} & \textbf{68.4} & \textbf{50.3} & \textbf{59.5} & 73.0 \\
 & \ourmethod{} ($\gamma=0.95$) & 62.7 & 48.5 & 38.5 & 31.7 & \textbf{69.4} & \textbf{68.4} & \textbf{50.4} & 56.0 & \textbf{76.0} \\
 & \ourmethod{} ($\gamma=0.99$) & 37.8 & 28.8 & 35.3 & 27.6 & \textbf{69.4} & \textbf{68.9} & \textbf{51.1} & \textbf{57.5} & 73.0 \\
\midrule
% --- CounterFact ---
\multirow{8}{*}{\rotatebox[origin=c]{90}{\textbf{CounterFact}}}
 & \llama & 1.2 & 1.0 & 0.3 & 0.6 & 69.5 & 69.3 & 50.7 & 58.0 & 73.5 \\
\cline{2-11}
 & \ourmethod{} ($\gamma=0.5$) & 45.5 & 31.5 & 5.7 & 7.2 & \textbf{68.5} & 50.4 & \textbf{51.4} & 50.0 & 43.5 \\
 & \ourmethod{} ($\gamma=0.6$) & 65.5 & 48.7 & 9.8 & 14.3 & \textbf{69.7} & 63.2 & \textbf{52.4} & \textbf{55.5} & \textbf{75.0} \\
 & \ourmethod{} ($\gamma=0.7$) & \textbf{75.7} & \textbf{57.3} & 15.5 & 20.4 & \textbf{69.4} & \textbf{66.8} & \textbf{51.7} & \textbf{55.0} & 72.5 \\
 & \ourmethod{} ($\gamma=0.8$) & \textbf{79.2} & \textbf{57.4} & 21.4 & 25.8 & \textbf{69.5} & \textbf{69.4} & \textbf{49.8} & \textbf{54.5} & \textbf{73.5} \\
 & \ourmethod{} ($\gamma=0.9$) & \textbf{79.4} & \textbf{55.9} & 38.4 & \textbf{32.4} & \textbf{69.3} & \textbf{67.5} & 49.5 & 54.0 & \textbf{76.5} \\
 & \ourmethod{} ($\gamma=0.95$) & 72.0 & 47.5 & \textbf{46.3} & \textbf{33.0} & \textbf{69.4} & \textbf{67.8} & \textbf{50.3} & \textbf{57.0} & \textbf{74.0} \\
 & \ourmethod{} ($\gamma=0.99$) & 51.6 & 27.7 & \textbf{45.3} & 26.8 & \textbf{69.4} & \textbf{68.2} & \textbf{51.7} & \textbf{56.0} & \textbf{76.5} \\
\midrule
% --- WikiBigEdit ---
\multirow{8}{*}{\rotatebox[origin=c]{90}{\textbf{WikiBigEdit}}}
 & \llama & 9.3 & 9.1 & 16.4 & 16.1 & 69.5 & 69.3 & 50.7 & 58.0 & 73.5 \\
\cline{2-11}
 & \ourmethod{} ($\gamma=0.5$) & 62.6 & 58.7 & 14.3 & 14.9 & \textbf{69.0} & \textbf{68.8} & \textbf{50.6} & 55.0 & 72.5 \\
 & \ourmethod{} ($\gamma=0.6$) & 66.5 & 60.8 & 17.4 & 19.1 & \textbf{69.3} & \textbf{68.2} & \textbf{51.4} & 53.0 & \textbf{75.0} \\
 & \ourmethod{} ($\gamma=0.7$) & \textbf{76.2} & \textbf{69.2} & 26.3 & 27.8 & \textbf{69.2} & \textbf{68.8} & \textbf{51.1} & 54.0 & \textbf{76.5} \\
 & \ourmethod{} ($\gamma=0.8$) & \textbf{77.2} & \textbf{72.1} & 21.2 & 24.4 & \textbf{69.4} & \textbf{69.1} & \textbf{50.4} & 55.0 & \textbf{76.5} \\
 & \ourmethod{} ($\gamma=0.9$) & \textbf{77.0} & \textbf{70.2} & 28.4 & 30.5 & \textbf{69.3} & \textbf{70.5} & \textbf{51.8} & 55.0 & \textbf{74.0} \\
 & \ourmethod{} ($\gamma=0.95$) & \textbf{76.9} & \textbf{68.9} & 23.4 & 27.3 & \textbf{69.2} & 62.6 & \textbf{51.2} & \textbf{57.5} & \textbf{74.5} \\
 & \ourmethod{} ($\gamma=0.99$) & 67.6 & 57.2 & \textbf{34.4} & \textbf{32.3} & \textbf{69.3} & 62.5 & \textbf{52.6} & \textbf{58.0} & 70.5 \\
\bottomrule
\end{tabular}
}%
\label{tab:ablation_gamma}
\end{table}

\end{document}